\documentclass{article}
\usepackage{arxiv}

\usepackage[utf8]{inputenc} 
\usepackage[T1]{fontenc}    
\usepackage[square,numbers]{natbib}
\usepackage{hyperref}       
\usepackage{url}            
\usepackage{booktabs}       
\usepackage{amsmath}
\usepackage{amsfonts}       
\usepackage{nicefrac}       
\usepackage{microtype}      

\usepackage{lipsum}         
\usepackage{graphicx}

\usepackage{doi}

\usepackage[table,xcdraw]{xcolor}
\usepackage[english]{babel}

\usepackage{cleveref}       

\usepackage{float}
\usepackage{amsthm}
\usepackage{mathrsfs}
\usepackage{amssymb}
\usepackage{pgfplots}
\usepackage{booktabs}
\usepackage{hhline}

\usepackage{doi}
\usetikzlibrary{spy}
\usepackage{tikz,pgfplots,pgf}
\usetikzlibrary{arrows,shapes,calc}
\usetikzlibrary{positioning}
\usepackage{neuralnetwork}
\usepackage[normalem]{ulem}

\newtheorem{example}{Example}
\newtheorem{theorem}{Theorem}
\newtheorem{proposition}{Proposition}
\newtheorem{definition}{Definition}

\newtheorem{corollary}{Corollary}
\newtheorem{lemma}{Lemma}
\newtheorem{notation}{Notation}

\pgfplotsset{compat=1.13}

\usepackage{xcolor, soul}
\sethlcolor{yellow}

\author{Ismaïl Baaj \\ Univ. Artois, CNRS, CRIL, F-62300 Lens, France \\  \href{baaj@cril.fr}{baaj@cril.fr}}

\title{Handling the inconsistency of systems of $\min\rightarrow$ fuzzy relational equations}

\hypersetup{
pdftitle={Handling the inconsistency of systems of min→ fuzzy relational equations},
pdfsubject={cs.AI},
pdfauthor={Ismaïl Baaj},
pdfkeywords={Fuzzy set theory ; Systems of fuzzy relational equations ; Chebyshev approximations},
}
\date{}

\begin{document}

\maketitle

\begin{abstract}
In this article, we study the inconsistency of systems of $\min-\rightarrow$ fuzzy relational equations. 
We give analytical formulas for computing the Chebyshev distances $\nabla = \inf_{d \in \mathcal{D}} \Vert \beta - d \Vert$ associated to  systems of  $\min-\rightarrow$ fuzzy relational equations of the form  $\Gamma \Box_{\rightarrow}^{\min} x = \beta$, where $\rightarrow$ is a residual implicator among the Gödel implication $\rightarrow_G$, the Goguen implication $\rightarrow_{GG}$ or Łukasiewicz’s
implication $\rightarrow_L$ and $\mathcal{D}$ is the set of second members of consistent systems defined with the same matrix $\Gamma$. The main preliminary result that allows us to obtain these formulas is that the Chebyshev distance $\nabla$ is the lower bound of the solutions  of a vector  inequality, whatever the residual implicator used. Finally, we show that, in the case of the $\min-\rightarrow_{G}$ system, the Chebyshev distance $\nabla$ may be an infimum, while it is always a minimum for  $\min-\rightarrow_{GG}$ and  $\min-\rightarrow_{L}$ systems.  
\end{abstract}


\section{Introduction}

Sanchez's pioneering work on solving systems of fuzzy relational equations based on $\max-\min$ composition \cite{sanchez1976resolution,sanchez1977} was followed by studies on solving systems based on $\max-T$ composition \cite{di1984fuzzy,miyakoshi1985solutions,pedrycz1982fuzzy} where $T$ is a t-norm, and studies on systems based on $\min-\rightarrow$ composition \cite{miyakoshi1985solutions,pedrycz1985applications,perfilieva2008system}, where $\rightarrow$ is the residual implicator associated to a t-norm. The authors of these studies
gave necessary and sufficient conditions  for systems of fuzzy relational equations  to be consistent.
However, it remains difficult to address the inconsistency of these systems. 

\noindent Recently, to address this issue, the author of   \cite{baaj2023maxmin} gives an explicit analytical formula to compute, using the $L_\infty$ norm, the Chebyshev distance $\Delta = \inf_{c \in \mathcal{C}} \Vert b - c \Vert$ associated to an inconsistent system of $\max-\min$ fuzzy relational equations $A \Box_{\min}^{\max} x = b$, where  $\mathcal{C}$ is the set of second members of consistent systems defined with the same matrix $A$. From this Chebyshev distance and in the case that the  system  $A \Box_{\min}^{\max} x = b$ is inconsistent, the author of \cite{baaj2023maxmin}  describes the structure of the set of  Chebyshev approximations of the second member $b$, which is formed by the vectors $c$ such that $\Vert b - c \Vert = \Delta$ and the system $A \Box_{\min}^{\max} x = c$ is a consistent system \cite{cuninghame1995residuation,li2010chebyshev}. He also establishes a structure theorem for the approximate solutions set of the inconsistent system, which is the set of solutions of  consistent systems $A \Box_{\min}^{\max} x = c$, where $c$ is a  Chebyshev approximation of $b$. From these results, the author introduces a learning paradigm of approximate weight matrices based on $\max-\min$ composition according to training data. \\
In \cite{baaj2023chebyshev}, the author extends some of these results  for  a system of max-product fuzzy relational equations and for a system of max-Łukasiewicz fuzzy relational equations.  He gives analytical formulas  for computing  the Chebyshev distance $\Delta$ associated to inconsistent systems based on max-product or max-Łukasiewicz compositions. For these two systems, the greatest approximate solution and the greatest Chebyshev approximation are directly computed according to the components of the matrix and those of the second member of the system,  so the learning paradigm introduced in \cite{baaj2023maxmin} can be used for the max-product and max-Łukasiewicz systems.

\noindent In this article, analogously to \cite{baaj2023chebyshev,baaj2023maxmin}, for   systems of  $\min-\rightarrow$ fuzzy relational equations of the form $\Gamma \Box_{\rightarrow}^{\min} x = \beta$ where $\Gamma$ is a fuzzy matrix of size $(m,n)$, $\beta$ is a column-vector of $m$ components  and  $\rightarrow$ is a residual implicator among the Gödel implication $\rightarrow_G$, the Goguen implication $\rightarrow_{GG}$ or Łukasiewicz’s
implication $\rightarrow_{L}$, we give explicit analytical formulas  for computing the Chebyshev distance $\nabla = \inf_{d \in \mathcal{D}} \Vert \beta - d \Vert$ where  $\mathcal{D}$ is the set of second members of consistent systems defined with the same matrix $\Gamma$.  Analogously to the application $F$ introduced in \cite{baaj2023chebyshev,baaj2023maxmin} for systems based on $\max-T$  composition, we introduce an idempotent, increasing and right-continuous application denoted $G$ for $\min-\rightarrow$  systems, see (\ref{eq::appG}). The application $G$ allows us to reformulate the necessary and sufficient conditions \cite{miyakoshi1985solutions,pedrycz1985applications,perfilieva2008system} for checking if a $\min-\rightarrow$ system  defined with a fixed matrix  and a given vector  used as the second member is a consistent system, see (\ref{qu:eq0}). As a first result (Theorem \ref{th1}),  this application $G$ allows us to show that  the Chebyshev distance $\nabla$ of a $\min-\rightarrow$ system $\Gamma \Box_{\rightarrow}^{\min} x = \beta$  is the lower bound of the solutions of a vector inequality expressed with this application $G$. Our (Theorem \ref{th1}) for $\min-\rightarrow$ systems, is the version of (Theorem 1) of Cuninghame-Green and Cechlárová for a $\max-\min$ system~\cite{cuninghame1995residuation}, which was  recently extended in \cite{baaj2023chebyshev} for $\max-T$ systems. Furthermore,  from this result,  we obtain the decomposition of  the Chebyshev distance $\nabla = \max_{1 \leq j \leq m} \nabla_j$  (Proposition \ref{prop:nablafunc}), where for all $j \in \{1,2,\dots,m\}$, each $\nabla_j$ is the  lower bound of the solutions of the $j$-th scalar inequality in the above-mentioned vector inequality, see (\ref{eq:gbotbtop}). \\
For each of the three systems of $\min-\rightarrow$ fuzzy relational equations (depending on whether $\rightarrow$ is the Gödel implication $\rightarrow_{G}$, the Goguen implication $\rightarrow_{GG}$ or Łukasiewicz’s
implication $\rightarrow_{L}$), we give an explicit formula for computing
 its associated Chebyshev distance $\nabla$ ((Theorem \ref{theorem:chebming}), (Theorem \ref{thformingg}) and  (Theorem \ref{th:luka}), respectively). We  also  prove the following important results: for $\min-\rightarrow_{G}$ system, the Chebyshev distance $\nabla$ may be an infimum (based on (Proposition \ref{prop:minGsmallestelem}) and showed in (Example \ref{ex:5counter})) while, for $\min-\rightarrow_{GG}$ and  $\min-\rightarrow_{L}$ systems, the Chebyshev distance $\nabla$ is always a minimum i.e. an attainable lower bound (see (Corollary \ref{cor:minGGmin}) and (Corollary \ref{cor:minLmin})).  For $\min-\rightarrow_{G}$ systems, we give a sufficient condition for $\nabla$ to be a minimum (see (Corollary \ref{cor:nablainE}) and (Corollary \ref{cor:chebnablainE})).
 For $\min-\rightarrow_{GG}$ systems, $\min-\rightarrow_{L}$ systems, or $\min-\rightarrow_{G}$ systems whose Chebyshev distance has been verified as a minimum, we can always compute the lowest Chebyshev approximation of the second member and an approximate solution (see  (Corollary \ref{cor:minGGcheb}),  (Corollary \ref{cor:minLcheb}) and (Corollary \ref{cor:chebnablainE})). However, for the $\min-\rightarrow_{G}$ systems whose Chebyshev distance has been verified as not being a minimum, the set of Chebyshev approximations of its second member is empty (Lemma \ref{lemma:minGchebempty}).

\noindent The article is structured as follows. In (Section \ref{sec:background}),  we remind the recent results of \cite{baaj2023chebyshev,baaj2023maxmin} on solving systems of $\max-T$ fuzzy relational equations, where $T$ is a continuous t-norm. In (Section \ref{sec:solvingminimp}), we introduce our main tools for studying the inconsistent systems of $\min-\rightarrow$ fuzzy relational equations: the application $G$ and its main properties, and  an important decomposition of the Chebyshev distance $\nabla$ associated to a system of $\min-\rightarrow$ fuzzy relational equations. 
Then, in the next three sections, we give the explicit analytical formulas for computing the Chebyshev distances $\nabla = \inf_{d \in \mathcal{D}} \Vert \beta - d \Vert$  in the case where the residual implicator is the Gödel implication   (Section \ref{sec:mingodel}), the Goguen implication (Section \ref{sec:mingoguen}) and  Łukasiewicz’s
implication (Section \ref{sec:minluka}). Finally, we conclude with some perspectives and we propose some applications of our work.

\section{Background}
\label{sec:background} 

\noindent In this section, we remind the necessary background for solving systems of fuzzy relational equations. We reuse some notations of \cite{baaj2023chebyshev,baaj2023maxmin} and denote by $T$  a continuous t-norm and $\mathcal{I}_T$ its associated residual implicator  \cite{klement2013triangular}. Finally, we present some recent  results of \cite{baaj2023chebyshev,baaj2023maxmin} on the inconsistency of systems of $\max-T$ fuzzy relational equations.

\subsection{Notations}

\noindent
 The set $[0,1]^{n\times m}$ is the set of matrices of size $(n,m)$ i.e., $n$ rows and $m$ columns, whose components are in $[0,1]$, thus, the set $[0,1]^{n\times 1}$ is the set of column vectors of $n$ components and $[0,1]^{1\times m}$ denotes  the set of row matrices of $m$ components.

\noindent The order relation $\leq$ on the set $[0,1]^{n\times m}$ is defined by:
\[ A \leq B \quad \text{iff we have} \quad  a_{ij} \leq b_{ij} \quad \text{ for all } \quad 1 \leq i \leq n, 1 \leq j \leq m,    \]
\noindent where $A=[a_{ij}]_{1 \leq i \leq n, 1 \leq j \leq m}$ and $B=[b_{ij}]_{1 \leq i \leq n, 1 \leq j \leq m}$. 

For $x,y,z,u,\delta \in [0,1]$, we put:
\begin{itemize}
    \item $x^+ = \max(x,0)$,
    \item $\overline{z}(\delta) = \min(z+\delta,1)$, 
    \item $\underline{z}(\delta) = \max(z-\delta,0) = (z-\delta)^+$
\end{itemize}
\noindent and we have the following equivalence in $[0,1]$: 
\begin{equation}\label{ineq:xyxbarybar}
    \mid x - y \mid \leq \delta \Longleftrightarrow \underline{x}(\delta) \leq y \leq \overline{x}(\delta).
\end{equation}

\noindent To a column-vector $b = [b_i]_{1 \leq  i \leq n}$ and a number $\delta \in [0,1]$, we associate two column-vectors:
    \begin{equation}\label{def:bstarhautbas} 
  \underline{b}(\delta) = [(b_i  - \delta)^+]_{1 \leq  i \leq n} \quad \text{and} \quad \overline{b}(\delta) = [\min(b_i  + \delta, 1)]_{1 \leq  i \leq n}.\end{equation}

\noindent Then, from (\ref{ineq:xyxbarybar}), we deduce for any $c = [c_i]_{1\leq i \leq n} \in [0,1]^{n \times 1}$: 
\begin{equation}\label{ineq:bbbar}
    \Vert b - c \Vert \leq \delta \Longleftrightarrow \underline{b}(\delta) \leq c \leq \overline{b}(\delta).
\end{equation}
\noindent where $ \Vert b - c \Vert = \max_{1 \leq i \leq n}\mid b_i-c_i\mid$.
\subsection{T-norms and their associated residual implicators }

\noindent
A triangular-norm  (t-norm, see  \cite{klement2013triangular}) is a map  $T: [0,1] \times [0,1] \mapsto [0,1]$, which satisfies:
\begin{itemize}
    \item[] $T$ is commutative:   $T(x,y) = T(y,x)$, 
    \item[] $T$ is associative:  $T(x, T(y, z)) = T(T(x, y), z)$, 
    \item[] $T$ is increasing  : $x \leq x' \quad \text{and} \quad y \leq y' \, \Longrightarrow \, T(x, y) \leq T(x', y')$,
    \item[]$T$ has $1$ as neutral element: $T(x, 1) = x$.
\end{itemize}

\noindent To the   t-norm $T$ is associated   the residual implicator  
${\cal I}_T : [0, 1 ] \times [0, 1] \rightarrow [0, 1 ] : (x, y) \mapsto {\cal I}_T(x, y) = \sup\{z \in [0, 1]\,\mid\, T(x, z) \leq y\}$.

\noindent
 For all $a, b   \in [0, 1]$, the main properties of the residual implicator ${\cal I}_T$ associated to a continuous t-norm $T$ are: 
\begin{itemize}
    \item   
${\cal I}_T(a, b) = \max\{z \in [0, 1]\,\mid\, T(a, z) \leq b\}$. Therefore, $T(a, {\cal I}_T(a, b)) \leq b$.
\item
  ${\cal I}_T$ is left-continuous and decreasing in its first argument as well as right-continuous and increasing in its second argument.
  \item
  For all $z\in [0, 1]$, we  have: 
  $$T(a, z) \leq b \Longleftrightarrow z \leq {\cal I}_T(a, b).$$
  \item We have 
  $\,b \leq {\cal I}_T(a, T(a, b))$.
\end{itemize}

\noindent The t-norm $\min$ denoted by $T_M$,  has a residual implicator $\mathcal{I}_{T_M}$ which is the Gödel implication:
\begin{equation}\label{eq:tnormmin}
    T_M(x,y) = \min(x, y) \quad ;  \quad {\cal I}_{T_M}(x,y) = x \underset{G}{\longrightarrow} y = \begin{cases}1 & \text{ if } x \leq y \\ \ y &\text{ if } x > y \end{cases}.
\end{equation}

The t-norm defined by the usual product   is denoted by $T_P$. Its associated residual implicator is the Goguen implication:
\begin{equation}\label{eq:tnormproduct}
    T_P(x,y) = x \cdot y \quad ;  \quad {\cal I}_{T_P}(x,y) = x \underset{GG}{\longrightarrow} y = \begin{cases}1 & \text{ if } x \leq y \\ \frac{y}{x} &\text{ if } x > y \end{cases}.
\end{equation}
Łukasiewicz's t-norm is denoted by $T_L$ and its associated residual implicator ${\cal I}_{T_L}$:
\begin{equation}\label{eq:tnormluka}
    T_{L}(x,y)= \max(x + y - 1, 0) = {(x + y - 1)}^{+} \quad ; \quad 
 {\cal I}_{T_L}(x,y)  = x \underset{L}{\longrightarrow} y = \min(1-x+y,1).
\end{equation}

\subsection{Solving systems of \texorpdfstring{$\max-T$}{max-T} fuzzy relational equations}

A system of $\max-T$ fuzzy relational equations based on a matrix $A=[a_{ij}] \in [0,1]^{n\times m}$   and a column-vector $b=[b_{i}] \in [0,1]^{n\times 1}$   is of the form:\begin{equation}\label{eq:sys}
    (S): A \Box_{T}^{\max} x = b,
\end{equation}
\noindent where $x = [x_j]_{1 \leq j \leq m} \in [0,1]^{m\times 1}$ is  an unknown vector and the matrix operator $\Box_{T}^{\max}$ is the matrix product that uses the continuous  t-norm $T$ as the product and $\max$  as the addition.

\noindent Using the vector \begin{equation}\label{eq:gr:sol}
    e = A^t \Box_{{\cal I}_T}^{\min} b,
\end{equation}
\noindent where $A^t$ is the transpose of $A$ and the matrix product $\Box_{{\cal I}_T}^{\min}$ uses the residual implicator  ${\cal I}_T$ (associated to $T$) as the product and $\min$ as the addition, we have the following equivalence  proved by  Sanchez for $\max-\min$ composition \cite{sanchez1976resolution}, and extended to $\max-T$ composition by Pedrycz \cite{pedrycz1982fuzzy,pedrycz1985generalized} and Miyakoshi and Shimbo \cite{miyakoshi1985solutions}:
\begin{equation}\label{eq:consiste}
    A \Box_T^{\max} x = b \text{ is consistent}\Longleftrightarrow A \Box_T^{\max} e = b. 
\end{equation}

\subsection{Chebyshev distance associated to the second member of a system of \texorpdfstring{$\max-T$}{max-T} fuzzy relational equations}

\noindent To the matrix $A$ and the second member $b$ of the system $(S)$ of $\max-T$ fuzzy relational equations, see (\ref{eq:sys}), is  associated  the set of  vectors $c = [c_i] \in [0,1]^{n \times 1}$ such that the system $A \Box_{T}^{\max} x = c$ is consistent:
\begin{equation}\label{def:setofsecondmembersB}
    \mathcal{C} = \{ c = [c_i] \in {[0,1]}^{n \times 1} \mid  A \Box_{T}^{\max} x = c \text{ is consistent} \}.
\end{equation}
\noindent This set allows us to define the Chebyshev distance associated to the second member $b$ of the system $(S)$.
 \begin{definition}\label{def:chebyshevdist}
The Chebyshev distance associated to the second member $b$ of the system $(S): A \Box_{T}^{\max}x = b$ is: 
\begin{equation}\label{eq:delta}
    \Delta = \Delta(A,b) =  \inf_{c \in \mathcal{C}} \Vert b - c \Vert 
    \end{equation}

    \end{definition}
\noindent where:
\[ \Vert b - c \Vert = \max_{1 \leq i \leq n}\mid b_i - c_i\mid.\]

\noindent and 
\begin{definition}
A Chebyshev approximation of the second member $b$ is an element  $c \in \mathcal{C}$, see (\ref{def:setofsecondmembersB}), such that $\Vert b- c\Vert = \Delta$.  
\end{definition}
The set of Chebyshev approximations of the second member $b$ is:
\begin{equation}\label{eq:setCb}
    \mathcal{C}_b = \{ c \in \mathcal{C} \mid \Vert b -c  \Vert = \Delta \}.  
\end{equation}

We have the following fundamental result,  proven for $\max-\min$ system in \cite{cuninghame1995residuation} and extended for $\max-T$ systems in \cite{baaj2023chebyshev}:

\begin{equation}\label{eq:deltaCUNINGH}
    \Delta =        \min\{\delta\in [0, 1] \mid \underline b(\delta) \leq F(\overline b(\delta))\}.
\end{equation}

\noindent which is based on the following application:
\begin{equation}\label{eq:F} 
    F :  [0, 1]^{n \times 1} \longrightarrow [0, 1]^{n \times 1} : c =[c_i] \mapsto F(c) = A \Box_{T}^{\max} (A^t \Box_{{\cal I}_T}^{\min}
    c) = [F(c)_i]
\end{equation}
\noindent where:
\begin{equation}\forall i \in \{1, 2, \dots, n\},\, F(c)_i = \max_{1 \leq j \leq m} T(a_{ij}, \min_{1 \leq k \leq n} 
{\cal I}_T(a_{kj}, c_k)).\end{equation}

\noindent By solving (\ref{eq:deltaCUNINGH}) in the case of a system of $\max-\min$ fuzzy relational equations $A \Box_{T_M}^{\max}x = b$, the author of \cite{baaj2023maxmin} gave the following analytical formula for computing the Chebyshev distance associated to its second member $b$: 
\begin{equation} \label{eq:Deltamaxmin}
\Delta_M = \max_{1 \leq i \leq n}  \min_{1 \leq j \leq m}\,\max[ (b_i - a_{ij})^+,  \max_{1 \leq k \leq n}\,  \,\sigma_G\,(b_i, a_{kj}, b_k)],\end{equation}
\noindent \text{ where }
\begin{equation}\label{eq:sigmaG}
    \sigma_G(x, y, z) = \min( \frac{(x - z)^+}{2},(y - z)^+).
\end{equation}

Similarly, for the case of a system of $\max$-product fuzzy relational equations $A \Box_{T_P}^{\max}x = b$,  the author of \cite{baaj2023chebyshev} gave the  following analytical formula for computing the Chebyshev distance associated to $b$:
\begin{equation} \label{eq:deltap}
\Delta_P = \max_{1 \leq i \leq n} \min_{1 \leq j \leq m}\, \max_{1 \leq k \leq n}\,\sigma_{GG}\,(a_{ij}, b_i, a_{kj}, b_k),     
\end{equation}
\noindent where
\begin{equation}\label{eq:sigmagg}
    \sigma_{GG}(u,x,y,z) = \max[(x-u)^+, \min(\varphi(u,x,y,z),(y-z)^+)] \end{equation}
    \noindent \text{ and } \begin{equation}\varphi(u,x,y,z) = \begin{cases}\frac{(x \cdot y - u\cdot z)^+}{u+y} &\text{ if } u > 0 \\ x &\text{ if } u = 0\end{cases}. \end{equation}

The author of \cite{baaj2023chebyshev} also gave the following analytical formula for computing Chebyshev distance associated to the second member $b$ of a system of $\max$-Łukasiewicz fuzzy relational equations $A \Box_{T_L}^{\max}x = b$:
\begin{equation} \label{eq:DeltaL}
\Delta_L = \max_{1 \leq i \leq n}  \min_{1 \leq j \leq m}\, \max_{1 \leq k \leq n}\,\sigma_{L }\,(1  - a_{ij}, b_i, a_{kj}, b_k),
\end{equation}
\noindent where
\begin{equation*}\label{eq:sigmaL}
  \sigma_L(u,x,y,z) = \min(x, \max(v^+, \frac{(v + y - z )^+}{2})) \text{ with } v = x+u-1.
\end{equation*}

For the $\max-\min$ system, the author of \cite{baaj2023maxmin} showed that the formula (\ref{eq:Deltamaxmin}) allows us to get directly its  greatest approximate solution and its greatest Chebyshev approximation of its second member (thus, the set $\mathcal{C}_b$ is non-empty and $\Delta_M$ is always a minimum). For $\max$-product and max-Łukasiewicz systems, the author of  \cite{baaj2023chebyshev} showed the same results using the formulas (\ref{eq:deltap}) and (\ref{eq:DeltaL}). From (\ref{eq:Deltamaxmin}),  the author of  \cite{baaj2023maxmin} also gave a description of the set $\mathcal{C}_b$ of Chebyshev approximations, see (\ref{eq:setCb}) and that of the approximate solutions set of the inconsistent system.

\section{\texorpdfstring{General results on solving systems of  $\min-\rightarrow$ fuzzy relational equations  where $\rightarrow$ is one of the implications among Gödel's implication, Goguen's or Łukasiewicz's}{General results on solving systems of fuzzy relational equations min→ where → is one of the implications among Gödel's implication, Goguen's or Łukasiewicz's}}
\label{sec:solvingminimp} 

\noindent In this section, we begin by reminding the solving of a  system of $\min-\rightarrow$ fuzzy relational equations, which is based on a matrix 
$\Gamma  = [\gamma_{ji}]_{1 \leq j \leq m, 1 \leq i \leq n}\in [0, 1]^{m \times n}$ of size $(m,n)$ and a column-vector of $m$ components $\beta = \begin{bmatrix} \beta_j \end{bmatrix}_{1 \leq j \leq m} \in [0 , 1]^{m\times 1}$:
\begin{equation}\label{eq:minsys}
    (\Sigma) : \Gamma  \Box_{\rightarrow}^{\min} x =  \beta,
\end{equation}
\noindent where $x$ is an unknown column-vector of $n$ components.\\
Using the vector $\epsilon = \Gamma^t  \Box_{T}^{\max} \beta$ (where $\Gamma^t$ is the transpose of the matrix $\Gamma$), we have the following equivalence \cite{miyakoshi1985solutions,pedrycz1985applications,perfilieva2008system}:
\begin{equation}\label{eq:equiv:cons}
     (\Sigma) \text{ is consistent } \Longleftrightarrow \Gamma  \Box_{\rightarrow}^{\min}\epsilon = \beta.
\end{equation}

\begin{example}
Let us use the Gödel's implication $\rightarrow_G$, see (\ref{eq:tnormmin}), the matrix $\Gamma = \begin{bmatrix}
    0.6 & 0.49 \\ 
    0.26 & 0.9 \\ 
\end{bmatrix}$ and the vector $\beta = \begin{bmatrix}
    0.58\\
    0.88
\end{bmatrix}.$\\ 
We have $\Gamma^t = \begin{bmatrix}
0.6 & 0.26 \\ 
0.49 & 0.9\\ 
\end{bmatrix}$. We compute $\epsilon = \Gamma^t \Box_{\min}^{\max} \beta = \begin{bmatrix}
    0.58 \\ 
    0.88
\end{bmatrix}$. The $\min-\rightarrow_G$ composition of the matrix  $\Gamma$ by the vector  $\epsilon$ is:
\[ \Gamma \Box_{\rightarrow_G}^{\min} \epsilon = \begin{bmatrix}
    0.58 \\ 
    0.88
\end{bmatrix} = \beta. \]
Therefore, the system  $\Gamma \Box_{\rightarrow_G}^{\min} x = \beta$ is consistent. 
\end{example}

\noindent Similarly to the application $F$, see (\ref{eq:F}), studied for systems of $\max-T$ fuzzy relational equations in \cite{baaj2023chebyshev,baaj2023maxmin}, we introduce the following application for systems of $\min-\rightarrow$ fuzzy relational equations: 
\begin{equation}\label{eq::appG}
    G : [0 , 1]^{m \times 1} \rightarrow  [0 , 1]^{m \times 1}  :  \xi \mapsto G(\xi)  = \Gamma  \Box_{{\cal I}_T}^{\min}
(\Gamma^t \Box_{T}^{\max} \xi),
\end{equation}
\noindent which allows us to reformulate the equivalence (\ref{eq:equiv:cons}) as:
\begin{equation}\label{qu:eq0}
     (\Sigma) \text{ is consistent } \Longleftrightarrow G(\beta) = \beta.
\end{equation}

The rest of this section is structured as follows. In (Subsection \ref{subsec:studyG}) we study the main properties of the function $G$. In (Subsection \ref{subsec:chebminfleche}) we define  the Chebyshev distance $\nabla = \inf_{d \in \mathcal{D}} \Vert \beta -d \Vert$  associated to  the second member $\beta$ of the system $(\Sigma)$  where $\mathcal{D}$ is the set formed by the second members of the consistent systems of $\min-\rightarrow$ fuzzy relational equations defined with the same matrix $\Gamma$.  We show in (Theorem \ref{th1}) that the Chebyshev distance $\nabla$ is  the lower bound of the solutions of a vector inequality expressed
with the application $G$. From this result,  we obtain the decomposition of  the Chebyshev distance $\nabla = \max_{1 \leq j \leq m} \nabla_j$  (Proposition \ref{prop:nablafunc}), where for all $j \in \{1,2,\dots,m\}$, each $\nabla_j$ is the  lower bound of the solutions of the $j$-th scalar inequality in the above-mentioned vector inequality, see (\ref{eq:gbotbtop}).

\subsection{\texorpdfstring{Study of the application $G$}{Study of the application G}}
\label{subsec:studyG}
\noindent
From the system $(\Sigma) : \Gamma  \Box_{{\cal I}_T}^{\min} x =  \beta$, see (\ref{eq:minsys}), where $\Gamma  = [\gamma_{ji}]_{1 \leq j \leq m, 1 \leq i \leq n}\in [0, 1]^{m \times n}$, let us compute the components of $G(\xi)$ for any 
  vector $\xi = [\xi_j]_{1\leq j \leq m} \in [0 , 1]^{m \times 1}$. \\For $1 \leq j \leq m$, we have:
$$G(\xi)_j = \min_{1 \leq i \leq n}\,
{\cal I}_T(\gamma_{ji} , \max_{1 \leq l \leq m}\, T(\gamma_{li} , \xi_l)).$$

The application $G$ has the following properties:
\begin{proposition}\label{propo1}
$$\forall \xi \in [0,1]^{m \times 1}, \quad \xi \leq G(\xi).$$
\end{proposition}
To prove the inequality $\xi_j \leq G(\xi)_j$ for each $j\in \{1 , 2 , \dots , m\}$, we rely on the inequality ${\cal I}_T(a ,  T(a , b))\geq b$.
\begin{proof}
\noindent
Let $j\in\{1 , 2 , \dots , m\}$, we have:
\begin{align}
    G(\xi)_j & = \min_{1 \leq i \leq n}\,
{\cal I}_T(\gamma_{ji} , \max_{1 \leq l \leq m}\, T(\gamma_{li} , \xi_l))\nonumber\\
& \geq \min_{1 \leq i \leq n}\,
{\cal I}_T(\gamma_{ji} ,  T(\gamma_{ji} , \xi_j))\nonumber\\
& \geq \min_{1 \leq i \leq n}\,\xi_j\nonumber\\
& =  \xi_j.\nonumber
\end{align}    
\end{proof}

\begin{proposition}\label{propo2}\mbox{}
The application $G$ is increasing ; given $\xi, \xi' \in [0 , 1]^{m \times 1}$  we have:
$$\xi \leq \xi'  \Longrightarrow G(\xi) \leq G(\xi').$$ 
\end{proposition}

\begin{proof}
The increasing property of the application $G$ is a consequence of the increasing property of the t-norm $T$ and that of its associated residual implicator ${\cal I}_T$ with respect to the second variable. It results that  the  following two applications are also increasing: 
$$[0 , 1]^{n \times 1} \rightarrow [0 , 1]^{m \times 1}:
c \mapsto \Gamma  \Box_{{\cal I}_T}^{\min}c \quad ; \quad 
[0 , 1]^{m \times 1}\rightarrow [0 , 1]^{n \times 1} :  
\xi  \mapsto \Gamma^t \Box_{T}^{\max} \xi.$$
 
\end{proof}
\begin{proposition}\label{propo3}\mbox{}
 The application $G$ is idempotent i.e., $G \circ G = G$.
\end{proposition}
To prove the idempotent property of $G$, we have to show that $\forall \xi \in [0,1]^{m \times 1}$, $G(G(\xi)) = G(\xi)$.  The proof is a consequence of the relation: $\forall \xi \in [0,1]^{m \times 1}, \quad \Gamma^t \Box_{T}^{\max} G(\xi) = \Gamma^t \Box_{T}^{\max} \xi$, which we will  deduce  from a property of the application $F$, see (\ref{eq:F}).

\begin{proof}
Let $\xi\in [0 , 1]^{m \times 1}$ be fixed and we put $c = \Gamma^t \Box_T^{\max}  \xi\in [0 , 1]^{n \times 1}$.  Thus, $ \Gamma^t \Box_T^{\max} x = c$ is a consistent system and $\xi$ one of its solution. We remind that the application $F$ of \cite{baaj2023chebyshev,baaj2023maxmin}, see (\ref{eq:F}), associated to the matrix $A:= \Gamma^t$ satisfies    $F(c) = c$. 

\noindent
We have: 
\begin{align}
 \Gamma^t \Box_T^{\max} G(\xi) & =  \Gamma^t \Box_T^{\max}( \Gamma \Box_{{\cal I}_T}^{\min}  (\Gamma^t \Box_T^{\max} \xi)) \nonumber\\
& =  \Gamma^t \Box_T^{\max}( \Gamma \Box_{{\cal I}_T}^{\min} c) \nonumber\\
& = F(c) \nonumber\\
& = c  =  \Gamma^t \Box_T^{\max}\xi. \nonumber
\end{align}
The equality $ \Gamma^t \Box_T^{\max} G(\xi)  =  \Gamma \Box_T^{\max}\xi$ leads to:
$$G(G(\xi)) =  \Gamma \Box_{{\cal I}_T}^{\min} ( \Gamma^t \Box_T^{\max} G(\xi)) = 
 \Gamma \Box_{{\cal I}_T}^{\min} ( \Gamma^t \Box_T^{\max}  \xi) = G(\xi).$$
\end{proof}

\begin{proposition}

The application $G$ is right-continuous.

\noindent If the application 
${\cal I}_T : [0 , 1]^2 \rightarrow [0  , 1]$
is  continuous, then $G : [0 , 1]^{m \times 1} \rightarrow [0 , 1]^{m \times 1}$ is also continuous.
\end{proposition}
\noindent The right-continuity of the application $G$ means that for any sequence $(\xi^{(k)})$ verifying:
$\forall k, \xi^{(k)} \geq \xi \text{ and }  (\xi^{(k)}) \rightarrow \xi$, we have $  (G(\xi^{(k))}) \rightarrow G(\xi)$. Note that, in $[0 , 1]^{m \times 1}$, the convergence of a sequence of vectors  $(\eta^{(k)})$ to a vector $\eta\in [0 , 1]^{m \times 1}$ is the convergence component by component:
$\forall j\in\{1 , 2 , \dots , m\},\, 
\eta^{(k)}_j \rightarrow \eta_j$.

\begin{proof}
The right-continuity of the application $G$ is a consequence of the continuity of the t-norm $T$ and  the right-continuity with respect to the second variable of the  residual implicator $\mathcal{I}_T$ associated to $T$. It is obtained by applying the general theorems on continuity \cite{rudin1976}, which also allow, in the case where ${\cal I}_T$ is continuous, to obtain the continuity of the application $G$.  
\end{proof}

In the following subsection, the application $G$ is used  for computing the Chebyshev distance associated to the second member of a system of $\min-\rightarrow$ fuzzy relational equations.

\subsection{\texorpdfstring{Chebyshev distance associated to the second member of a system of $\min-\rightarrow$ fuzzy relational equations}{Chebyshev distance associated to the second member of a system of min→ fuzzy relational equations}}
\label{subsec:chebminfleche}

\noindent In this subsection, we begin by defining the set $\mathcal{D}$  formed by the second members of the consistent systems of $\min-\rightarrow$ fuzzy relational equations defined with the same matrix $\Gamma$.  For $\max-T$ systems, an analogous set denoted $\mathcal{C}$ was defined, see  (\ref{def:setofsecondmembersB}). 

\noindent To the matrix $\Gamma$, let us  associate  the set of  vectors $d = [d_j] \in [0,1]^{m \times 1}$ such that the system $\Gamma  \Box_{\rightarrow}^{\min} x = d$ is consistent:
\begin{equation}\label{def:setD}
    \mathcal{D} = \{ d = [d_j] \in {[0,1]}^{m \times 1} \mid  \Gamma  \Box_{\rightarrow}^{\min} x = d \text{ is consistent} \}.
\end{equation}
\noindent 
Note that the equivalence (\ref{qu:eq0}) allows us to obtain a second definition of the set $\mathcal{D}$:
\begin{equation}\label{def:setofsecondmembersBeta0}
    \mathcal{D} = \{ d = [d_j] \in {[0,1]}^{m \times 1} \mid  G(d) = d   \}.
\end{equation}
Using the idempotence property  of the application $G$ (Proposition \ref{propo3}), we also deduce:
\begin{equation}\label{qu:eq00}
\forall \beta \in [0 , 1]^{m \times 1}\,,\, G(\beta) \in 
\mathcal{D}.
\end{equation}
The set $\mathcal{D}$ is non-empty: the vector $d\in [0,1]^{m \times 1}$ whose components are equal to $1$,  satisfies the equality 
$G(d) = d$ and  then $d\in \mathcal{D}$.   \\   
This set allows us to define:
 \begin{definition}\label{def:chebyshevdistMinfleche}
The Chebyshev distance associated to the system $(\Sigma): \Gamma  \Box_{\rightarrow}^{\min} x = \beta$ is: 
\begin{equation}\label{eq:nabla}
    \nabla = \nabla(\Gamma,\beta) =  \inf_{d \in \mathcal{D}} \Vert \beta - d \Vert 
    \end{equation}
    \noindent where:
\[ \Vert \beta - d \Vert = \max_{1 \leq j \leq m}\mid \beta_j - d_j\mid.\]
    \end{definition}
and

\begin{definition}
A Chebyshev approximation of the second member $\beta$ is an element  $d \in \mathcal{D}$, see (\ref{def:setofsecondmembersBeta0}), such that $\Vert d - \beta \Vert = \nabla$.  
\end{definition}
The set of Chebyshev approximations of the second member $\beta$ is:
\begin{equation}\label{eq:setDb}
    \mathcal{D}_\beta = \{ d \in \mathcal{D} \mid \Vert d - \beta \Vert = \nabla \}.  
\end{equation}

\noindent In the case of a system of $\max-T$ fuzzy relational equations, the authors of \cite{baaj2023chebyshev,cuninghame1995residuation} showed that its associated Chebyshev distance $\Delta$ is the lower bound of the solutions of a vector inequality involving the application   $F$,  (see \ref{eq:deltaCUNINGH}). \\
\noindent
Similarly, for a system of $\min-\rightarrow$  fuzzy relational equations, we will prove in (Theorem \ref{th1}) that the Chebyshev distance (\ref{eq:nabla}) is the lower bound of the solutions  of a vector inequality involving the application $G$.\\   

We rely on the following sets:
\begin{notation}
   \begin{equation}\label{qu:ETot}
 E = \{\delta\in [0, 1] \mid G(\underline \beta(\delta))  \leq \overline{\beta}(\delta)\}.
\end{equation}

For all $1 \leq j \leq m$:
\begin{equation}\label{eq:setEj}
E_j = \{\delta\in[0 , 1] \,\mid\, G(\underline{\beta}(\delta))_j \leq \overline{\beta}(\delta)_j\}, \quad  \nabla_j = \inf E_j. 
\end{equation} 
\begin{equation}\label{eq:setAj}
 {\cal A}_j=\{i \in \{1 , 2 , \dots , n\}\,\mid \,
\gamma_{ji} > 0 \}, 
\end{equation}

\end{notation}

\noindent 
We need the following lemma:
\begin{lemma}\label{lemdec}
We have:
\begin{enumerate}
 \item  $E = \bigcap_{1 \leq j \leq m} E_j.$
 \item 
 for all $\delta , \delta' \in [0 , 1]$ and $j\in \{1 , 2 , \dots , m\}$ , we have:
\[\delta \in E_j \quad \text{and} \quad \delta \leq \delta' 
 \,\Longrightarrow \,\delta'\in E_j.\]
\end{enumerate}
\end{lemma}
\begin{proof}
The proof of the first statement is easy: 
for any $\delta \in[0 , 1]$, we have: 
\[ \delta \in E 
\,\Longleftrightarrow \,
G(\underline{\beta}(\delta))  \leq \overline{\beta}(\delta) 
\,\Longleftrightarrow \,
\forall j\in \{1 , \dots , m\}  \, ,\, G(\underline{\beta}(\delta))_j  \leq \overline{\beta}(\delta)_j
\,\Longleftrightarrow \,
\forall   j\in \{1 , \dots , m\}\, ,\, \delta \in E_j.\]
Let $\delta \in E_j$ and $\delta'\in [0 , 1]$ satisfy the inequality $\delta \leq \delta'$. Then, using the increasing (Proposition \ref{propo2}) of $G$, we obtain: 
\[G(\underline{\beta}(\delta'))_j \leq 
G(\underline{\beta}(\delta))_j
 \leq
\overline{\beta}(\delta)_j
\leq
\overline{\beta}(\delta')_j.\]
Thus, we obtain $G(\underline{\beta}(\delta'))_j \leq \overline{\beta}(\delta')_j$, i.e., $\delta' \in E_j$.
\end{proof}

The Chebyshev distance $\nabla$ (Definition \ref{def:chebyshevdistMinfleche}) of a $\min-\rightarrow$ system $(\Sigma) : \Gamma  \Box_{\rightarrow}^{\min} x =  \beta$, see (\ref{eq:minsys}) is:
\begin{theorem}\label{th1} 
We have:
\begin{equation}\label{eq:nablatotal}
   \nabla =        \inf\{\delta\in [0, 1] \mid G(\underline \beta(\delta))  \leq \overline{\beta}(\delta)\} = \inf E .
\end{equation}
\end{theorem}
\begin{proof}
We put 
$\nabla' = \inf E.$ Let us show that we have  
$\nabla = \nabla'$ in two steps:
\begin{enumerate}
    \item $\nabla \leq \nabla'$: we have to show:
    \[ \forall \delta \in E,\, \nabla \leq    \delta. \]
Let  $\delta \in E$, we put $d = G(\underline \beta(\delta))$.  From (\ref{qu:eq00}), we deduce that $d \in \mathcal{D}$. 
Then, we obtain:
\[ \Vert \beta - G(\underline \beta(\delta)) \Vert = \Vert \beta - d \Vert \geq \nabla. \]
From (Proposition \ref{propo1}), we have $\underline \beta(\delta) \leq G(\underline \beta(\delta))$ and $\delta \in E$. So, we deduce the double inequality 
\[ \underline \beta(\delta) \leq G(\underline \beta(\delta)) \leq \overline \beta(\delta).  \]
Using the inequality 
(\ref{ineq:bbbar}), we conclude that: 
\[ \Vert \beta - d \Vert = 
\Vert \beta - G(\underline \beta(\delta)) \Vert \leq \delta \]
and finally $\nabla \leq  
\Vert \beta - d \Vert \leq \delta$. 
\item $\nabla' \leq \nabla$: we have to show:
\[ \forall d\in {\cal D}\,,\,
\nabla' \leq \Vert \beta - d \Vert.  \] 
Let $d \in {\cal D}$, we put 
$\delta = \Vert \beta - d \Vert$. From (\ref{ineq:bbbar}), we deduce: 
\[\underline \beta(\delta) \leq  d \leq \overline \beta(\delta).\]
Using  that $G$ is increasing  (Proposition \ref{propo2}) and  the equivalence  (\ref{qu:eq0}), we obtain:
\[ G(\underline \beta(\delta)) \leq G(d) = d \leq \overline \beta(\delta), \]
so $\delta \in E$ and by definition of $\nabla'$, we have  
$\nabla' \leq  \delta = \Vert \beta - d \Vert$. 
\end{enumerate}
\end{proof}
\noindent (Theorem \ref{th1}) states that $\inf_{d \in \mathcal{D}} \Vert \beta - d \Vert = \nabla = \inf E$. In fact, in the following sections, we will see that, whatever the residual implicator used among the Gödel implication, the Goguen implication and the Łukasiewicz implication,  we have: 
$$ \nabla = \min_{d \in \mathcal{D}} \Vert \beta - d \Vert   \Longleftrightarrow \nabla = \min E.$$
\noindent 
From the above lemma and theorem, we deduce    the following useful decomposition of   the Chebyshev distance $\nabla$:  
\begin{proposition}\label{prop:nablafunc}
    \begin{equation}
        \nabla = \max_{1\leq j \leq m} \nabla_j
    \end{equation}
\noindent where:
\begin{equation}\label{eq:gbotbtop}
 \nabla_j = \inf\{\delta\in [0 , 1]\,\mid \,G(\underline{\beta}(\delta))_j \leq \overline{\beta}(\delta)_j\} = \inf E_j. 
\end{equation}
\end{proposition}
\begin{proof}
We will prove the equality 
$\nabla = \max_{1\leq j \leq m} \nabla_j$ in the following two steps:
\begin{enumerate}
    \item $\max_{1\leq j \leq m} \nabla_j \leq \nabla$ ; we have to show: 
\[ \forall \delta \in E,\,
\max_{1\leq j \leq m} \nabla_j \leq \delta. \]
Using (Lemma \ref{lemdec}) the equality $E = \bigcap_{1 \leq j \leq m} E_j$, we deduce that for any $\delta\in E$, we have:
\[   \forall j \in\{1 , 2 , \dots , m\}\,\, \nabla_j = \inf E_j \leq \delta. \]
Thus, $\max_{1\leq j \leq m} \nabla_j \leq \delta.$
\item $\nabla \leq \max_{1\leq j \leq m} \nabla_j$  

\noindent
To rely on  the properties of a lower bound, we use the well-known equivalence:
$$\nabla \leq \max_{1\leq j \leq m} \nabla_j
\,\Longleftrightarrow \,
\forall \varepsilon > 0 \,,\, 
\nabla <  \max_{1\leq j \leq m} \nabla_j + \varepsilon. $$
Let $\varepsilon >  0$. For any 
$j\in \{1  , 2 , \dots , m\}$, there is an element $\delta_j \in E_j$ such that 
\[  \delta_j < \nabla_j + \varepsilon.\]
We put $\delta = \max_{1 \leq j \leq m} \delta_j$. Then, for any 
$j\in\{1 , 2 , \dots , m\}$, we get:
\[ \delta_j \in E_j \quad \text{and} \quad \delta_j \leq \delta. \]
Using  (Lemma \ref{lemdec}), we conclude that $\delta \in E$ and: \[  \nabla \leq \delta= \max_{1 \leq j \leq m} \delta_j <\max_{1\leq j \leq m} (\nabla_j + \varepsilon) = \max_{1 \leq j \leq m} \nabla_j + \varepsilon. \] 
\end{enumerate}
\end{proof}

As we will see in the next sections, the formula for computing $\nabla_j$, involves the value $1 - \beta_j$ (thus using the $j$-th component $\beta_j$ of the second member $\beta$), whatever the residual implicator used among the Gödel implication, the Goguen implication and the Łukasiewicz implication. 

\begin{lemma}\label{lemma:nablajvsunmoinsbeta}
For any 
$j\in \{1  , 2 , \dots , m\}$, we have:
 \begin{enumerate}
\item $1 - \beta_j \in E_j$, 
\item $\nabla_j \leq 1 - \beta_j$. 

\end{enumerate}
In particular, we have:
$$ \beta_j = 1 \Longrightarrow \nabla_j = 0.$$
\end{lemma}
\begin{proof}
For $\delta = 1 - \beta_j$, we have $\overline{\beta}(\delta)_j = 1 
$, so the inequality 
$G(\underline{\beta}(\delta))_j \leq \overline{\beta}(\delta)_j = 1$ is satisfied and therefore we have $1 - \beta_j \in E_j$.

\noindent
As we proved $1 - \beta_j \in E_j$, the inequality $\nabla_j \leq 1 - \beta_j$ is a consequence of the definition $\nabla_j = \inf E_j$. 
\end{proof}

In the following sections, for each of the three systems of $\min-\rightarrow$ fuzzy relational equations  (depending on whether $\rightarrow$ is Gödel's implication, Goguen's implication or Łukasiewicz's implication), we give an explicit formula for computing the Chebyshev distance associated to the second member of the system.

\section{\texorpdfstring{Chebyshev distance associated to the second member of systems based on $\min-\rightarrow_G$ composition}{Chebyshev distance associated to the second member of a system of min→gödel fuzzy relational equations}}
\label{sec:mingodel} 

In this section, our purpose is to give an explicit formula (Theorem \ref{theorem:chebming}) for computing the Chebyshev distance $\nabla$,  see (Definition \ref{def:chebyshevdistMinfleche}), associated to the second member of  a system of $\min-\rightarrow_G$ fuzzy relational equations $(\Sigma) : \Gamma  \Box_{\rightarrow_{G}}^{\min} x =  \beta$, see (\ref{eq:minsys}), where $\rightarrow_G$ is the Gödel implication, see (\ref{eq:tnormmin}). The main preliminary results ((Proposition \ref{proposition:minusunmoinsbetamaxthetazeta}), (Corollary \ref{corollary:c1}) and (Proposition \ref{enplus1})) needed for establishing  the formula  (Theorem \ref{theorem:chebming}) follow from the key result stated in  (Lemma \ref{lemma:equivmin}).

\noindent We remind from (Proposition \ref{prop:nablafunc}) that $\nabla = \max_{1\leq j \leq m} \nabla_j$ where $\nabla_j = \inf E_j$ and the set $E_j$ is defined in (\ref{eq:setEj}).   As illustrated at the end of the section by (Example \ref{ex:5counter}),  the set $E_j$ does not necessarily admit a minimum element, i.e. $\nabla_j$ does not necessarily belong to $E_j$. In (Proposition \ref{prop:minGsmallestelem}), we characterize when $E_j$ admits a minimum element i.e., $\nabla_j \in E_j$ and deduce a sufficient condition for $\nabla$ being a minimum element of $E$ (Corollary \ref{cor:nablainE}) i.e., $\nabla \in E$. \\ Therefore, as we will see in (Example \ref{ex:5counter}), in the case of a $\min-\rightarrow_{G}$ system, the Chebyshev distance $\nabla$ may be an infimum.  For $\min-\rightarrow_{G}$ systems whose $\nabla$ has been verified as a minimum, we can always compute the lowest Chebyshev approximation of the second member and an approximate solution (Corollary \ref{cor:chebnablainE}). However, for the $\min-\rightarrow_{G}$ systems whose Chebyshev distance has been verified as not being a minimum, the set of Chebyshev approximations of its second member is empty (Lemma \ref{lemma:minGchebempty}).

We use the following notations:

\begin{notation}
For $1 \leq j \leq m, 1 \leq i \leq n $, to each coefficient $\gamma_{ji}$ of the matrix $\Gamma$ of the system $(\Sigma)$, see (\ref{eq:minsys}), we associate:
\begin{itemize}
    \item $V(j,i) = \big\{ l \in \{1,2,\dots,m\} \mid  \gamma_{ji} \leq \gamma_{li} \big\}$,
    \item $\theta_{ji} = \max_{l\in V(j , i)} (\beta_l -\gamma_{ji})$,
    \item $\zeta_{ji} = \max_{1 \leq l \leq m}\sigma_G(\beta_l , \gamma_{li} , \beta_j)$.
\end{itemize}
\noindent where $\sigma_G(x,y,z) = \min(\dfrac{(x - z)^+}{2} , (y - z)^+)$ was defined in (\ref{eq:sigmaG}) from \cite{baaj2023maxmin}. 
\end{notation}

\noindent To give the formula for computing $\nabla_j$ for $j \in \{1,2,\dots,m\}$, see (\ref{eq:gbotbtop}), we will study the solving of the involved inequality $G(\underline{\beta}(\delta))_j \leq \overline{\beta}(\delta)_j$.

\noindent The following result gives an equivalent definition of the function $\sigma_G$ introduced in \cite{baaj2023maxmin}, see (\ref{eq:sigmaG}),   which is adapted to our needs:

\begin{lemma}\label{lemma:godel:solvineq}
For all $x , y , z , \delta\in [0 , 1]$, we have:
$$\min(y , \underline{x}(\delta)) \leq \overline{z}(\delta) \,
\Longleftrightarrow \,
\sigma_G(x , y   , z)\leq \delta.$$
\end{lemma}
\begin{proof}
We have:
$$y - \overline{z}(\delta) = 
\max(y - z - \delta , y - 1).$$
\begin{align}
\underline{x}(\delta) - \overline{z}(\delta)
&= \max(\underline{x}(\delta) - z - \delta , \underline{x}(\delta) -1)\nonumber\\
& = \max[\max(x - \delta , 0) - z - \delta , \underline{x}(\delta) -1]\nonumber\\
& = \max[\max(x - \delta - z - \delta, -z - \delta)   , \underline{x}(\delta) -1]\nonumber\\
& = \max(x -   z - 2\delta, -z - \delta   , \underline{x}(\delta) -1).\nonumber\\
\end{align}
We remark that: $y - 1 \leq 0,\, -z - \delta \leq 0, \text{ and } \underline{x}(\delta) -1 \leq 0$.

We have:
\begin{align}
\min(y , \underline{x}(\delta)) \leq \overline{z}(\delta)
&  \Longleftrightarrow\,
\min(y - \overline{z}(\delta) , \underline{x}(\delta) - \overline{z}(\delta)) \leq 0 \nonumber\\
& \Longleftrightarrow\,
y - \overline{z}(\delta) \leq 0 \,\text{ or }\,
\underline{x}(\delta) - \overline{z}(\delta)) \leq 0
\nonumber\\
& \Longleftrightarrow\,
y - z - \delta \leq 0 \,\text{ or }\,
x - z - 2\delta \leq 0
\nonumber\\
& \Longleftrightarrow\,
(y - z)^+ \leq \delta   \,\text{ or }\,
\dfrac{(x - z)^+}{2} \leq \delta  
\nonumber\\
& \Longleftrightarrow\,
\min(\dfrac{(x - z)^+}{2} , (y - z)^+)  \leq \delta  
\nonumber\\
& \Longleftrightarrow\,
\sigma_G(x , y , z) = \min(\dfrac{(x - z)^+}{2} , (y - z)^+)  \leq \delta.  
\nonumber\\
\end{align}
\end{proof}
\begin{example}
    Let $x=0.6, y = 0.4$ and $z=0.2$. We have $\min(y,x) > z$. We compute $\delta = \sigma_G(x,y,z) = 0.2$, then the inequality $\min(y , \underline{x}(\delta)) \leq \overline{z}(\delta)$ is satisfied. 
\end{example}

 We establish the following important lemma, which will provide important information in (Proposition \ref{proposition:minusunmoinsbetamaxthetazeta}) on   $\nabla_j$,  in the case where $\nabla_j < 1-\beta_j$:

\begin{lemma}\label{lemma:equivmin}
Let $j\in \{1 , 2 , \dots , m\}$. 
For all $\delta \in [0 , 1]$ such that $\delta  < 1 - \beta_j$, we have the following equivalence between these two statements:
\begin{enumerate}
\item $G(\underline{\beta}(\delta))_j \leq  \overline{\beta}(\delta)_j$,
\item there exists $i\in \{1 , 2 , \dots, n\}$ such that:
$$\gamma_{ji} > 0, \quad  \theta_{ji}   < \delta \quad  \text{ and } \quad   \zeta_{ji}  \leq  \delta.$$ 
\end{enumerate}
\end{lemma}  
\begin{proof}\mbox{}\\
$\Longrightarrow$: \\
We have: 
$$G(\underline{\beta}(\delta))_j = \min_{1 \leq i \leq n}(\gamma_{ji}  \rightarrow _G (\max_{1 \leq l \leq m}\min(\gamma_{li} , \underline{\beta}(\delta)_l)))
\leq \overline{\beta}(\delta)_j = \beta_j + \delta < 1.$$
We deduce that there exists an index $i\in \{1 , 2 , \dots  n\}$ such that: 
$$\gamma_{ji}  \rightarrow _G (\max_{1 \leq l \leq m}\min(\gamma_{li} , \underline{\beta}(\delta)_l)) =
G(\underline{\beta}(\delta))_j   \leq \overline{\beta}(\delta)_j = \beta_j + \delta < 1$$
which implies 
$$\gamma_{ji} > \max_{1 \leq l \leq m}\min(\gamma_{li} , \underline{\beta}(\delta)_l)\quad \text{ and } \quad \max_{1 \leq l \leq m}\min(\gamma_{li} , \underline{\beta}(\delta)_l) \leq  \beta_j + \delta $$
so 
\begin{equation}\label{eq:soproofLemm4}
    \gamma_{ji} > 0\text{ and }\forall l\in \{1 , 2 , \dots , m\}\,,\, \gamma_{ji} >  \min(\gamma_{li} , \underline{\beta}(\delta)_l)\text{ and }\min(\gamma_{li} , \underline{\beta}(\delta)_l) \leq  \beta_j + \delta.
\end{equation}

\noindent By (Lemma \ref{lemma:godel:solvineq}), we know that for all $l\in\{1 
 , 2 , \dots ,m\}$ , we have: 
 $$\min(\gamma_{li} , \underline{\beta}(\delta)_l) \leq  \beta_j + \delta = \overline{\beta}(\delta)_j
 \Longleftrightarrow
 \sigma_G(\beta_l , \gamma_{li} , \beta_j) \leq \delta$$
so (\ref{eq:soproofLemm4}) implies $
    \zeta_{ji} = \max_{1 \leq l \leq m} \sigma_G(\beta_l , \gamma_{li} , \beta_j) \leq \delta.$

\noindent To establish the inequality   $\theta_{ji} < \delta$, we will show that:
$$\forall l\in V(j , i) \, , \,   \beta_l -   \gamma_{ji} <  \delta.$$
Let  $l\in V(j , i)$ i.e.,  $\gamma_{ji} \leq \gamma_{li}$ and by (\ref{eq:soproofLemm4}) we have: 
$\gamma_{ji} >  \min(\gamma_{li} , \underline{\beta}(\delta)_l)$, so   
$$\min(\gamma_{li} , \underline{\beta}(\delta)_l) = \underline{\beta}(\delta)_l  < \gamma_{ji}.
$$
Then, we have:
$$\underline{\beta}(\delta)_l   -  \gamma_{ji} = \max(\beta_l - \delta , 0) -  \gamma_{ji}  = 
\max(\beta_l -   \gamma_{ji}  - \delta  , - \gamma_{ji}) < 0 $$
so $\beta_l -   \gamma_{ji}  < \delta$ and finally 
$\theta_{ji} = \max_{l\in V(j , i)} \beta_l -   \gamma_{ji}  < \delta$.
\noindent

$\Longleftarrow$:\\ Let $i\in \{1 , 2 , \dots, n\}$ such that:
$$ \gamma_{ji} > 0, \quad  \theta_{ji}   < \delta, \quad   \zeta_{ji}  \leq  \delta.$$
We have:
$$G(\underline{\beta}(\delta))_j = \min_{1 \leq i' \leq n}(\gamma_{ji' }  \rightarrow _G (\max_{1 \leq l \leq m}\min(\gamma_{li' } , \underline{\beta}(\delta)_l))
\leq  \gamma_{ji}  \rightarrow _G (\max_{1 \leq l \leq m}\min(\gamma_{li } , \underline{\beta}(\delta)_l)).$$
We will show that we have:
$$\gamma_{ji}  \rightarrow _G (\max_{1 \leq l \leq m}\min(\gamma_{li } , \underline{\beta}(\delta)_l)  \leq \overline{\beta}(\delta)_j.$$
From the general  equality: $x \rightarrow_G \max(y , z) = \max(x \rightarrow_G y , x \rightarrow_G z)$, we deduce:
$$\gamma_{ji}  \rightarrow _G (\max_{1 \leq l \leq m}\min(\gamma_{li} , \underline{\beta}(\delta)_l))  =  \max_{1 \leq l \leq m}  \gamma_{ji}  \rightarrow _G  \min(\gamma_{li} , \underline{\beta}(\delta)_l).$$
We will establish \begin{equation}\label{eq:estblish}
    \forall l \in\{1 , 2 , \dots , m\},\, \gamma_{ji}  \rightarrow _G  \min(\gamma_{li} , \underline{\beta}(\delta)_l) \leq 
\overline{\beta}(\delta)_j   = \beta_j + \delta.
\end{equation}

\noindent
Let $l \in\{1 , 2 , \dots , m\}$. We distinguish the following two cases:

$\bullet\,$ We suppose that $l\in V(j , i)$ i.e., $  \gamma_{ji} \leq \gamma_{li}$. We then have: 
$$\beta_l -  \gamma_{li}   \leq \beta_l  - \gamma_{ji} \leq \theta_{ji} < \delta.$$
We deduce: 
$$\beta_l  - \gamma_{li} - \delta \leq  \beta_l  - \gamma_{ji} - \delta < 0  
\quad  \text{ and } \quad \gamma_{li} \geq \gamma_{ji}  > 0.$$
So:
$$\underline{\beta}(\delta)_l   -    \gamma_{li} = \max(\beta_l  - \gamma_{li} - \delta, -  \gamma_{li})  < 0$$
which leads to
$ \min(\gamma_{li} , \underline{\beta}(\delta)_l)   = \underline{\beta}(\delta)_l$.
On the other hand, $\beta_l  - \gamma_{ji} - \delta  < 0  
\quad  \text{ and } \quad \gamma_{ji} > 0 $  lead to
$$\underline{\beta}(\delta)_l   -    \gamma_{ji}    = 
\max(\beta_l  - \gamma_{ji} - \delta, -  \gamma_{ji})  < 0$$
so
$\underline{\beta}(\delta)_l < \gamma_{ji}$ and we obtain 
$$\gamma_{ji}  \rightarrow _G  \min(\gamma_{li} , \underline{\beta}(\delta)_l)  =  \min(\gamma_{li} , \underline{\beta}(\delta)_l).$$
By taking into account   $\sigma_G(\beta_l, \gamma_{li}, \beta_j) \leq \zeta_{ji} \leq \delta$, we deduce from (Lemma \ref{lemma:godel:solvineq}):  
$$\min(\gamma_{li} , \underline{\beta}(\delta)_l) \leq \overline{\beta}(\delta)_j$$
so 
$$\gamma_{ji}  \rightarrow _G  \min(\gamma_{li} , \underline{\beta}(\delta)_l) \leq \overline{\beta}(\delta)_j = \beta_j + \delta.$$
$\bullet\,$ If we suppose $l\in V(j , i)^c$, i.e, $  \gamma_{ji} >  \gamma_{li} $  and taking into account that  $\gamma_{li}  \geq \min(\gamma_{li},\underline{\beta}(\delta)_l)$, we then have:
$$\gamma_{ji}  \rightarrow _G  \min(\gamma_{li} , \underline{\beta}(\delta)_l) = \min(\gamma_{li} , \underline{\beta}(\delta)_l).$$
As by hypothesis we have $\sigma_G(\beta_l, \gamma_{li}, \beta_j) \leq \zeta_{ji} \leq \delta$, we deduce from (Lemma \ref{lemma:godel:solvineq}):
$$\min(\gamma_{li} , \underline{\beta}(\delta)_l) \leq \overline{\beta}(\delta)_j.$$
So:
$$\gamma_{ji}  \rightarrow _G  \min(\gamma_{li} , \underline{\beta}(\delta)_l) = \min(\gamma_{li} , \underline{\beta}(\delta)_l)  \leq \overline{\beta}(\delta)_j = \beta_j + \delta.$$
We have established (\ref{eq:estblish}) and therefore we have proven $G(\underline{\beta}(\delta))_j \leq  \overline{\beta}(\delta)_j$. 
\end{proof}
We illustrate this result:
\begin{example}\label{ex:illuduresPropositiontaree}
Let us use the matrix $\Gamma = \begin{bmatrix}
    0.6 & 0.49 \\ 
    0.26 & 0.9 \\ 
\end{bmatrix}$ and the vector $\beta = \begin{bmatrix}
    0.1\\
    0.4
\end{bmatrix}$. For $j=1$, we have:

\begin{align*}
    G(\underline{\beta}(\delta))_1 &= \min_{1 \leq i \leq 2}(\gamma_{1i}  \rightarrow _G (\max_{1 \leq l \leq 2}\min(\gamma_{li} , \underline{\beta}(\delta)_l)))\\
&=\min(0.6  \rightarrow _G (\max_{1 \leq l \leq 2}\min(\gamma_{l1}, \underline{\beta}(\delta)_l)), 0.49  \rightarrow _G (\max_{1 \leq l \leq 2}\min(\gamma_{l2}, \underline{\beta}(\delta)_l))) 
\end{align*}
\noindent We compute:
\begin{itemize}
    \item $\max_{1 \leq l \leq 2}\min(\gamma_{l1}, {\beta}_l) = \max(\min(0.6,0.1),\min(0.26,0.4))=0.26$,
    \item $\max_{1 \leq l \leq 2}\min(\gamma_{l2}, {\beta}_l) = \max(\min(0.49,0.1), \min(0.9,0.4))=0.4$
\end{itemize}
\noindent and then: $0.6 \rightarrow_G 0.26=0.26$ and $0.49 \rightarrow_G 0.4=0.4$. So the inequality is not satisfied for $\delta = 0$. 

\noindent For solving the inequality, we rely on: $V(1,1) = \{1\}$ and $V(1,2) = \{ 1,2\}$ and we compute: 
\begin{itemize}
    \item $\theta_{11} = \beta_1 - \gamma_{11} = 0.1 - 0.6 = -0.5$, 
    \item $\theta_{12} = \max(\beta_1 - \gamma_{12}, \beta_2 - \gamma_{12})=\max(0.1-0.49,0.4-0.49)=-0.09$,
    \item $\zeta_{11} = \max(\sigma_G(0.1,0.6,0.1),\sigma_G(0.4,0.26,0.1))= 0.15$, 
    \item $\zeta_{12} = \max(\sigma_G(0.1, 0.49,0.1), \sigma_G(0.4, 0.9,0.1)) = 0.15$.  
\end{itemize}
\noindent We take $\delta = 0.15$ and we have $\gamma_{11} > 0$, $\theta_{11} < \delta$ and $\zeta_{11} = \delta$.

\noindent Then we observe that $ G(\underline{\beta}(\delta))_1 = \min(0.6 \rightarrow_G 0.25, 0.49 \rightarrow_G 0.25) = 0.25 \leq \overline{\beta}(\delta)_1 = 0.25$. So the inequality is solved with $\delta = 0.15$. 
\end{example}

\begin{proposition}\label{proposition:minusunmoinsbetamaxthetazeta}
Let $j\in \{1 , 2 , \dots , m\}$. Suppose that $\nabla_j < 1 - \beta_j$. Then, there exists $i^0\in\{1 , 2 , \dots , n\}$ such that:  
$$\gamma_{ji^0  } > 0 \quad \text{ and } \quad \nabla_j \geq \max(\theta_{ji^0 }, \zeta_{ji^0 }).$$
  \end{proposition}

\begin{proof}
The inequality $\nabla_j < 1 - \beta_j$ and the properties of the  lower bound $\nabla_j = \inf E_j$ allow us to find a sequence $(\delta_k)$ in $[0 , 1]$ such that:
$$\forall  k \,,\,\delta_k \in E_j\, \, \text{i.e}\, \,G(\underline{\beta}(\delta_k))_j \leq \overline{\beta}(\delta_k)_j,  \quad (\delta_k) \rightarrow 
 \nabla_j,  \quad  \forall  k \,,\, \nabla_j \leq \delta_k < 1 - \beta_j.$$
By applying (Lemma \ref{lemma:equivmin}) to each  $\delta_k$, we get  an integer $i_k \in \{1 , 2 , \dots , n\}$ such that: 
$$\gamma_{ji_k  } > 0, \quad  \theta_{ji_k  }   < \delta_k, \quad   \zeta_{ji_k  } \leq  \delta_k.$$
The sequence of integers $k \mapsto i_k$ takes its values in the finite set $\{1 , 2 , \dots , n\}$, so by (Theorem 3.6) of \cite{rudin1976}, it admits a subsequence  $k \mapsto i_{\alpha(k)}$, which is {\it stationary} (see the proof of (Theorem 3.6) of \cite{rudin1976}). Therefore, we have an integer $i^0\in  \{1 , 2 , \dots , n\}$ such that:
$$\forall k \,, \, i_{\alpha(k)} = i^0.$$
Then, we deduce:
$$\gamma_{ji^0  } > 0, \quad \forall k \,, \,       \theta_{ji^0  }   < \delta_{\alpha(k)}, \quad  \zeta_{ji^0  } \leq  \delta_{\alpha(k)}.$$
By passage to the limit when $k \rightarrow \infty$, we obtain:
$$  \theta_{ji^0  }   \leq   \nabla_j, \quad   \zeta_{ji^0 }  \leq   \nabla_j.$$
 \end{proof}
 As a direct consequence, we have:
\begin{corollary}\label{corollary:c1}
Let $j\in \{1 , 2 , \dots , m\}$. Suppose that $\nabla_j < 1 - \beta_j$. Then, the set ${\cal A}_j$, see (\ref{eq:setAj}), is non empty.
\end{corollary}

From  (Proposition \ref{proposition:minusunmoinsbetamaxthetazeta})  and (Corollary \ref{corollary:c1}), we establish a formula for computing $\nabla_j$ if $\nabla_j < 1 - \beta_j$:

 \begin{proposition}\label{enplus1} Suppose that $\nabla_j < 1 - \beta_j$, then we have:
$$\nabla_j = \min_{i\in {\cal A}_j}\, 
\max(\theta_{ji} , \zeta_{ji}).$$
\end{proposition}
\begin{proof}
We put $\tau_j = \min_{i\in {\cal A}_j}\, 
\max(\theta_{ji} , \zeta_{ji})$ and we will show 
$\nabla_j = \tau_j$ in the following two steps:

\begin{itemize}
    \item $\nabla_j \geq \tau_j$: by (Proposition \ref{proposition:minusunmoinsbetamaxthetazeta}), we have an index $i^0 \in \{1,2,\dots,n\}$ such that:
    $$\gamma_{ji^0 } > 0, \quad \max(\theta_{ji^0} , \zeta_{ji^0})   \leq   \nabla_j .$$
\noindent Therefore $i^0 \in \mathcal{A}_j$ and $\max(\theta_{ji^0},\zeta_{ji^0}) \leq \nabla_j$ and we deduce: 
\[ \tau_j \leq \max(\theta_{ji^0},\zeta_{ji^0}) \leq \nabla_j. \]
\item $\nabla_j = \tau_j$. To prove this equality, we proceed by contradiction:   we suppose that  $\nabla_j \neq \tau_j$. We then have $\nabla_j > \tau_j$. The set $\mathcal{A}_j$ being non-empty (Corollary \ref{corollary:c1}), let $i \in \mathcal{A}_j$ such that $\tau_j = \max(\theta_{ji}, \zeta_{ji})$, so we have:
\begin{equation}
    \gamma_{ji} > 0, \quad \max(\theta_{ji}, \zeta_{ji}) = \tau_j < \nabla_j.
\end{equation}
We take a number $\delta$ verifying $\tau_j < \delta < \nabla_j$, then $\delta$ satisfy:
\[ \delta < \nabla_j < 1 - \beta_j,\]
\noindent and the index $i \in \mathcal{A}_j$ verify:
\[ \gamma_{ji} > 0, \quad \max(\theta_{ji}, \zeta_{ji}) = \tau_j < \delta < 1 - \beta_j. \]
\noindent By (Lemma \ref{lemma:equivmin}), we deduce that $\delta \in E_j$, i.e.,  $G(\underline{\beta}(\delta))_j \leq  \overline{\beta}(\delta)_j$), so $\nabla_j= \inf E_j \leq \delta$, which is a contradiction. 
\end{itemize}

\end{proof}

We then give the explicit analytical formula for computing the Chebyshev distance $\nabla$ (see (Definition \ref{def:chebyshevdistMinfleche})   and  (Proposition \ref{prop:nablafunc})), associated to the second member $\beta$ of a system of $\min-\rightarrow_G$ fuzzy relational equations $(\Sigma) : \Gamma  \Box_{\rightarrow_{G}}^{\min} x =  \beta$, see (\ref{eq:minsys}):
\begin{theorem}\label{theorem:chebming}
$$\nabla = \max_{1 \leq j \leq m} \nabla_j \quad  \text{where for }  j \in \{1,2,\dots,m\}, \quad  \nabla_j = 
\min(1 - \beta_j , \tau_j)$$
$$\text{ with }  \quad 
\tau_j = \min_{i\in {\cal A}_j}\, \max(\theta_{ji}, \zeta_{ji}) \quad \text{ and the convention }   \min_\emptyset = 1.$$ 
\end{theorem}

Following (Proposition \ref{prop:nablafunc}), we have to prove that $\text{for all }  j \in \{1,2,\dots,m\}, \nabla_j =  \inf E_j$ is equal to $\min(1 - \beta_j , \tau_j)$.

\begin{proof}
We distinguish the following cases:
\begin{itemize}
    \item If $\beta_j = 1$, we deduce from (Lemma \ref{lemma:nablajvsunmoinsbeta}) that 
    $\nabla_j = 0 = \min(1-\beta_j,\tau_j).$
    \item Suppose that $\beta_j < 1$. We then have:
    \begin{itemize}
        \item If $\mathcal{A}_j = \emptyset$, then by (Lemma \ref{lemma:nablajvsunmoinsbeta}) and (Corollary \ref{corollary:c1}), we have $\nabla_j = 1 - \beta_j$ and $\tau_j=1$ (by the convention). Therefore, $\nabla_j   = \min(1-\beta_j,\tau_j),$
        \item If $\mathcal{A}_j \neq \emptyset$, then there exists an index $i \in \mathcal{A}_j$ such that $\tau_j = \max(\theta_{ji}, \zeta_{ji})$.\\ Let us prove    the inequality $\nabla_j \leq \tau_j$ by contradiction.\\
        Assume that we have  $\tau_j < \nabla_j$. We know from (Lemma \ref{lemma:nablajvsunmoinsbeta}) that $\nabla_j \leq 1 - \beta_j$. Let $\delta$ be a number verifying  
$\tau_j < \delta < \nabla_j$. Then, we have:
$$\gamma_{ji} > 0 \quad \text{ and } \quad 
\tau_j = \max(\theta_{ji}, \zeta_{ji}) < \delta 
< \nabla_j \leq 1 - \beta_j.$$
By (Lemma \ref{lemma:equivmin}), we deduce $\delta \in E_j$, so $\nabla_j= \inf E_j \leq \delta$ which is a contradiction.
    \end{itemize}
\end{itemize}
\noindent
To summarize, we know that:
$$\nabla_j\leq \min(1 - \beta_j, \tau_j)
\quad \text{ and by (\ref{enplus1})} \quad \nabla_j < 1 - \beta_j\Longrightarrow \nabla_j = \tau_j$$
so $\nabla_j = \min(1 - \beta_j , \tau_j)$.
\end{proof}
We illustrate this theorem:
\begin{example}
We continue (Example \ref{ex:illuduresPropositiontaree}). For $j=1$, we compute  $\tau_1 = \min_{i \in \mathcal{A}_1}\max(\theta_{1i}, \zeta_{1i})$ where $\mathcal{A}_1 = \{ 1,2 \}$. We obtain $\tau_1 = 0.15$. As  $1 - \beta_1 = 0.9$, we have $\nabla_1 = \min(0.9,0.15) = 0.15$. 
\end{example}

As we will see in   (Example \ref{ex:5counter}), $\nabla_j = \inf E_j$
\textit{does not always belong to the set $E_j$} (the set $E_j$ is defined in (\ref{eq:setEj})) i.e., $E_j$ does not necessarily admit a minimum element. Our aim in what follows is to characterize the case in which $E_j$ admits a minimum element i.e., $\nabla_j \in E_j$. 

\noindent
We know from (Lemma \ref{lemma:nablajvsunmoinsbeta}) that $1- \beta_j \in E_j$ and  $\nabla_j \leq 1 - \beta_j$. We are left with the case where  $\nabla_j < 1 - \beta_j$. Thus, we have $\beta_j <  1$. We will need the following lemma:
\begin{lemma}\label{l1}
Suppose that we have $\beta_j < 1$, then for all $i\in{\cal A}_j$, we have $\zeta_{ji} < 1 - \beta_j.$
\end{lemma}
\begin{proof}
\begin{align}
\zeta_{ji} &= \max_{1 \leq l \leq m}
\sigma_G(\beta_l  , \gamma_{li} , \beta_j)\nonumber\\
&= \max_{1 \leq l \leq m}
\min(\dfrac{(\beta_l - \beta_j)^+}{2}  , (\gamma_{li} -    \beta_j)^+)\nonumber\\
&\leq \max_{1 \leq l \leq m}
\dfrac{(\beta_l - \beta_j)^+}{2}\nonumber\\ 
&\leq \dfrac{1 - \beta_j}{2} < 1 - \beta_j. \nonumber 
\end{align}
\end{proof}
 
We use the following sets:
\begin{notation}For $j \in \{1,2,\dots,m\}$: \\
    \begin{equation}\label{eq:fjset}
    \mathcal{F}_j = \{ i \in \mathcal{A}_j \mid \theta_{ji} < \zeta_{ji} \} \quad \text{and} \quad  \widetilde \nabla_j = \min_{i \in \mathcal{F}_j} \zeta_{ji} \quad \text{ with the convention }   \min_\emptyset = 1.
\end{equation}
\end{notation}

\begin{lemma}\label{l2}Suppose we have $\beta_j < 1$. 
We have: $\widetilde \nabla_j\in  E_j$.    
\end{lemma}
\begin{proof}
We distinguish the following two cases:
\begin{itemize}
    \item  If $\mathcal{F}_j = \emptyset$, then $\widetilde \nabla_j  = 1 \in E_j$.
    \item If $\mathcal{F}_j \neq \emptyset$, let $i\in \mathcal{F}_j$ such that 
\[ \widetilde \nabla_j  = \zeta_{ji}.\]
From $i \in \mathcal{F}_j$ and (Lemma \ref{l1}), we deduce: 
\[ \gamma_{ji} > 0 \quad \text{and} \quad 
\theta_{ji} < \zeta_{ji} < 1 - \beta_j\]
\end{itemize}

\noindent
By taking $\delta = \zeta_{ji}< 1 - \beta_j$, we get:
\[ \gamma_{ji} > 0 \,,\quad \theta_{ji} < \delta \quad \text{and} \quad 
  \zeta_{ji} = \delta   \]
We deduce from  (Lemma \ref{lemma:equivmin}) that we have: 
\[ \widetilde \nabla_j  = \zeta_{ji} = \delta  \in E_j. \]
\end{proof}
 The characterization of the case where $E_j$ admits a minimum element, here $\nabla_j$,  is given by:
\begin{proposition}\label{prop:minGsmallestelem}
Suppose we have $\nabla_j < 1 - \beta_j$. Then, we have:
\begin{enumerate}
\item $\nabla_j \leq \widetilde\nabla_j$, 
  \item $\nabla_j \in E_j
 \,\Longleftrightarrow\,
   \nabla_j = \widetilde \nabla_j.$ 
\end{enumerate}   
\end{proposition}
We may have $\nabla_j \notin E_j$ (see in (Example \ref{ex:5counter})). 
\begin{proof}
We know from (Lemma \ref{l2}) that $\widetilde \nabla_j\in E_j$. Then the definition of $\nabla_j = \inf E_j$ implies that $\nabla_j \leq \widetilde\nabla_j.$

\noindent
To prove the equivalence of the second statement, since from (Lemma \ref{l2}) we know that 
$\widetilde\nabla_j \in E_j$, the equality 
$\nabla_j = \widetilde\nabla_j$ implies that $\nabla_j \in E_j.$

\noindent
To prove the implication $\Longrightarrow$, let us first remark that the hypothesis $\nabla_j \in E_j$ means that 
$\nabla_j = \min E_j$. 

\noindent
As we suppose  $\nabla_j < 1 - \beta_j$, we deduce from 
(Lemma \ref{lemma:equivmin}) that there is an index $i^0\in{\cal A}_j$ such that 
\[ \theta_{ji^0} < \nabla_j \quad \text{and} \quad \zeta_{ji^0} \leq \nabla_j.\]
Let us show that we have $i^0 \in\mathcal{F}_j$ i.e.,  
$\theta_{ji^0} < \zeta_{ji^0}$ by contradiction.

\noindent
Suppose we have $\zeta_{ji^0} \leq \theta_{ji^0}$. Let us take a number $\delta$ verifying:  
\[ \theta_{ji^0} < \delta < \nabla_j.\]
We then have:
\[\gamma_{ji^0} > 0 \quad \text{and} \quad \zeta_{ji^0} \leq \theta_{ji^0} < \delta < \nabla_j < 1 -\beta_j. \]
By applying (Lemma \ref{lemma:equivmin}) to $\delta$,we obtain that
$\delta\in E_j$, so $\nabla_j \leq \delta$ which is a contradiction.

\noindent 
Finally, we have:
 \[\widetilde\nabla_j = \min_{i\in\mathcal{F}_j} \zeta_{ji} \leq 
\zeta_{ji^0} \leq  \nabla_j\]
so $\nabla_j = \widetilde\nabla_j$.
\end{proof} 

We immediately deduce from (Lemma \ref{lemdec}) and (Proposition \ref{prop:nablafunc}): 
\begin{corollary}\label{cor:nablainE}
Assume that for all $j \in \{1,2,\dots,m\}$, we have either $\nabla_j = 1 - \beta_j$ or $\nabla_j = \widetilde \nabla_j$, then we have $\nabla \in E$. 
\end{corollary}

We then deduce:
\begin{corollary}\label{cor:chebnablainE}
 If $\nabla \in E$, then $G(\underline{\beta}(\nabla))$ is the lowest Chebychev approximation for $\beta$. Moreover, $\xi = \Gamma^t \Box_{\min}^{\max} \underline{\beta}(\nabla)$ is an approximate solution for the system 
$(\Sigma) : \Gamma \Box_{\rightarrow_{G}}^{\min} x = \beta$.   
\end{corollary}
(From the second part of the following proof, one can see that if   $\nabla = \min_{d \in \mathcal{D}} \Vert \beta - d \Vert$ then $\nabla \in E$).
\begin{proof}
By definition of the application $G$, we have 
$G(\underline{\beta}(\nabla)) = \Gamma \Box_{\rightarrow_{G}}^{\min} \xi$, so
$G(\underline{\beta}(\nabla)) \in{\cal D}$ is a consistent second member. By definition of $\nabla$ (Definition \ref{def:chebyshevdistMinfleche}), we have $\Vert \beta - G(\underline{\beta}(\nabla))\Vert \geq \nabla$.

\noindent
From (Proposition \ref{propo1}) and the equality $\nabla = \min E$, we deduce the double inequality:
\[ \underline{\beta}(\nabla) \leq  G(\underline{\beta}(\nabla)) \leq \overline{\beta}(\nabla). \]
By applying the inequality (\ref{ineq:bbbar}), we get: 
$$\Vert \beta - G(\underline{\beta}(\nabla))\Vert \leq \nabla\quad \text{so} \quad  \Vert \beta - G(\underline{\beta}(\nabla))\Vert = \nabla.$$
We have proven that  $G(\underline{\beta}(\nabla))$ is a Chebyshev approximation of $\beta$ and $\xi$ is an approximate solution for the system 
$(\Sigma)$. 

\noindent
Let $c \in {\cal D}_\beta$ be any   Chebyshev approximation of $\beta$, i.e., $G(c) = c$ and $\Vert  \beta - c \Vert = \nabla$. By reusing the inequality (\ref{ineq:bbbar}), we get:
\[ \underline{\beta}(\nabla) \leq  c  \leq \overline{\beta}(\nabla). \]
As $G$ is an increasing application, see (Proposition \ref{propo2}), and applying this property  to the inequality $
\underline{\beta}(\nabla) \leq  c$ , we get:
\[ G(\underline{\beta}(\nabla)) \leq G(c) = c. \]

\end{proof}

\noindent With the help of the following example, we show that $E_j$ does not necessarily admit a minimum  element and $\nabla \notin E$:

\begin{example}\label{ex:5counter}
We use $\Gamma = \begin{bmatrix}
    0.41 & 0.07\\
0.29 & 0.31
\end{bmatrix}$ and $\beta = \begin{bmatrix}
    0.88\\
0.46
\end{bmatrix}$. 
We compute:
\begin{itemize}
    \item $V(1,1) = \{ 1 \}$, $\theta_{11} = 0.47$, $\zeta_{11} = 0$,
    \item $V(1,2) = \{  1,2\}$, $\theta_{12} = 0.81$, $\zeta_{12} = 0$, 
     \item $V(2,1) = \{ 1,2  \}$, $\theta_{21} = 0.59$, $\zeta_{21} = 0$, 
      \item $V(2,2) = \{  2\}$, $\theta_{22} = 0.15$, $\zeta_{22} = 0$.
\end{itemize}
\noindent
We deduce that:
$$\tau_1 = 0.47, \quad 1 - \beta_1 = 0.12, \quad \nabla_1 = \min(0.12 , 0.47) = 0.12 = 1 - \beta_1\in E_1$$
$$\tau_2 = 0.15, \quad 1 - \beta_2 = 0.54, \quad \nabla_2 = \min(0.54 ,   0.15) = 0.15 < 1 - \beta_2.$$

We have $\nabla = \max(\nabla_1,\nabla_2) = 0.15$. We note that for all $j \in \{1,2\}$, we have ${\cal F}_j = \emptyset$. But, we remark that:

\noindent
$\bullet\,$ We have $\nabla_1 \in E_1$  although ${\cal F}_1 = \emptyset$ because $\nabla_1 = 1 - \beta_1$.

\noindent
$\bullet\,$ We have $\nabla_2 \notin E_2$  because 
$\nabla_2 < 1 - \beta_2$ and ${\cal F}_2 = \emptyset$ which implies $\nabla_2 < \widetilde{\nabla_2} = 1$.

\noindent $\bullet\,$ Therefore, we have $\nabla = \max(\nabla_1,\nabla_2) = \nabla_2 \notin E_2$. So, as  $E = E_1 \cap E_2$,  we have $\nabla \notin E$ i.e., the set $E$ does not admit a minimum element. 
\end{example}
As we have seen, the Chebyshev distance $\nabla$ may be an infimum. Therefore, as one might expect, we have: 
\begin{lemma}\label{lemma:minGchebempty}
If $\nabla \notin E$, then the set of Chebyshev approximations $\mathcal{D}_\beta$, see (\ref{eq:setDb}), of the second member $\beta$ of the system $\Gamma \Box_{\rightarrow_G}^{\min}x = \beta$ is empty.  
\end{lemma}
\begin{proof}
We prove that  $\mathcal{D}_\beta$ is empty by contradiction. \\
Assume that we have an element $d \in \mathcal{D}_\beta$ i.e., $G(d) = d$ and $\Vert d - \beta \Vert = \nabla$. 
From the inequality (\ref{ineq:bbbar}), we deduce:
\begin{equation*}
    \underline{\beta}(\nabla) \leq d \leq \overline{\beta}(\nabla),
\end{equation*}
\noindent Then, by applying the increasing application $G$ to the inequality $\underline{\beta}(\nabla) \leq d$, we have:
\begin{equation*}
    G(\underline{\beta}(\nabla)) \leq G(d) = d \leq \overline{\beta}(\nabla) 
\end{equation*}
Then, we have $G(\underline{\beta}(\nabla)) \leq \overline{\beta}(\nabla)$ i.e., $\nabla \in E$, which is contradiction. 
\end{proof}

In the next two sections, we will show that for  $\min-\rightarrow_{GG}$  and $\min-\rightarrow_L$ systems, the sets $E_j$ always admit a minimum element, i.e., we always have $\nabla_j \in E_j$ and therefore by (Lemma \ref{lemdec}) and (Proposition \ref{prop:nablafunc}), we have  $\nabla \in E$. 

\section{\texorpdfstring{Chebyshev distance associated to the second member of systems based on $\min-\rightarrow_{GG}$ composition}{Chebyshev distance associated to the second member of a system of min→goguen fuzzy relational equations}}
\label{sec:mingoguen}
In this section,  we study the Chebyshev distance $\nabla$ associated to the second member of a system of $\min-\rightarrow_{GG}$ fuzzy relational equations $(\Sigma) : \Gamma  \Box_{\rightarrow_{GG}}^{\min} x =  \beta$, see (\ref{eq:minsys}), where $\rightarrow_{GG}$ is the Goguen implication, see (\ref{eq:tnormproduct}). 

The proofs of some (but not all) of the statements in this section are similar to those of some of the results in the previous section.
To avoid repeating the same arguments in the proofs, we have retained the notations $\nabla_j , \theta_{ji}$ and $\zeta_{ji}$ from the previous section, whose definitions are here adapted to the case of the $\min - \rightarrow_{GG}$ system. They differ from those given for the $\min - \rightarrow_{G}$ system in the previous section.

\noindent The formula for computing the Chebyshev distance associated to the second member of a system of $\min - \rightarrow_{GG}$ fuzzy relational equations  is established in (Theorem \ref{thformingg}). Thanks to the choice of the same notations ($\nabla_j, \theta_{ji}$ and $\zeta_{ji}$)  for the $\min - \rightarrow_{G}$ system  and the $\min - \rightarrow_{GG}$ system, (Theorem \ref{thformingg}) is formulated in the same way as  (Theorem \ref{th1}) for the  $\min - \rightarrow_{G}$ system.

\noindent We remind that we have $\nabla = \max_{1\leq j \leq m} \nabla_j$ where $\nabla_j = \inf E_j$ (Proposition \ref{prop:nablafunc}) and the set $E_j$ is defined in (\ref{eq:setEj}). Unlike the  $\min - \rightarrow_{G}$ system, which we studied in the previous section, we will show that in the case of the  $\min - \rightarrow_{GG}$ system, for all $j \in \{1,2,\cdots,m\}$, we always have $\nabla_j \in E_j$, i.e., $\nabla_j$ is always the minimum element of the set $E_j$ although the Goguen's implication $\rightarrow_{GG}$ is not continuous with respect to its two variables. By applying (Lemma \ref{lemdec}) and (Proposition \ref{prop:nablafunc}), we obtain  $\nabla = \min E$ (Theorem \ref{th:minGGisMin}), where $E$ is defined in (\ref{qu:ETot}). Therefore,  in the case of a $\min-\rightarrow_{GG}$ system, we will show that  the Chebyshev distance $\nabla$ (Definition \ref{def:chebyshevdistMinfleche}) is always a minimum, see (Corollary \ref{cor:minGGmin}). In fact, we can always compute the lowest Chebyshev approximation of the second member
and an approximate solution of the system $(\Sigma)$, see (Corollary \ref{cor:minGGcheb}).

We begin by introducing a new function based on the function $\sigma_{GG}$, see (\ref{eq:sigmagg}), which was introduced in \cite{baaj2023chebyshev}:
\begin{equation}\label{eq:func:p}
    P(u,x,y,z) = \begin{cases}
        0 & \text{ if } u = 0,\\
        0& \text{ if } y = 0,\\
        \sigma_{GG}(\dfrac{u}{y}, x, 1, z)  & \text{ if } u > 0 \text{ and } y > 0 
    \end{cases}.
\end{equation}
\noindent We have:
\begin{equation}\label{eq:egalitesursigmagg}
    \sigma_{GG}(\dfrac{u}{y} , x , 1 , z) = \max[(x - \dfrac{u}{y})^+ , \min(\dfrac{(x y -  u z)^+}{u + y} , 1 - z)].
\end{equation}

\noindent We use the following notations:
\begin{notation}
For $1 \leq j \leq m, 1 \leq i \leq n$, to each coefficient $\gamma_{ji}$ of the matrix $\Gamma$ of the system $(\Sigma)$, see (\ref{eq:minsys}), we associate:
\begin{itemize}
    \item $W(j , i) = \{l\in \{1 , 2 , \dots , m\}\,\mid \, \gamma_{ji} \leq \gamma_{li} \,\text{ and }\, \gamma_{li} > 0\}$,
    \item $\theta_{ji} = \max_{l\in W(j , i)} (\beta_l - \dfrac{\gamma_{ji}}{\gamma_{li}})$ with the convention $\max\limits_{\emptyset} = 0$,
    \item $\zeta_{ji} = \max_{1 \leq l \leq m}P(\gamma_{ji} , \beta_l , \gamma_{li} , \beta_j)$.
\end{itemize}
\end{notation}
\noindent The function $P$, see (\ref{eq:func:p}), allows us to solve a particular inequality: 
\begin{lemma}\label{l3} For all $u ,x , y , z , \delta\in [0 , 1]$ with $u > 0$  and $\delta < 1 - z$,  we have:
$$ \dfrac{ y \, \underline{x}(\delta)}{u}    \leq \overline{z}(\delta) \,
\Longleftrightarrow \,
P(u , x , y   , z)\leq \delta.$$
\end{lemma}
\begin{proof}\mbox{}\\
We remind that if $y > 0$,     the following equivalence was established in (Proposition 2) of  \cite{baaj2023chebyshev}:  
\begin{equation}\label{eq:trucpiqueachebmaxprod}
    \underline{x}(\delta) \leq \dfrac{u}{y} (1 \rightarrow_{GG} \overline{z}(\delta)) \,\Longleftrightarrow\,
\sigma_{GG}(\dfrac{u}{y} , x , 1 , z) \leq \delta.
\end{equation}

$\Longrightarrow$:\\
If $y = 0$, we have 
$$P(u , x , y , z) = 0 \leq  \delta.$$
If $y > 0$, from the inequality $\overline{z}(\delta) < 1$ , we deduce  : 
$$\underline{x}(\delta) \leq \dfrac{u}{y} \overline{z}(\delta) = \dfrac{u}{y} (1 \rightarrow_{GG} \overline{z}(\delta)).$$
So, by (\ref{eq:trucpiqueachebmaxprod}), we have $P(u , x , y , z) = \sigma_{GG}(\dfrac{u}{y} , x , 1 , z)\leq \delta$.

\noindent
$\Longleftarrow$:\\
Suppose that 
$P(u , x , y , z) \leq \delta$. 

$\bullet$  
If $y = 0$, the inequality $\dfrac{ y \, \underline{x}(\delta)}{u}    \leq \overline{z}(\delta)$ is trivial.

$\bullet$
If $y > 0$, since $\overline{z}(\delta) < 1$, the inequality to prove is:  
$$\underline{x}(\delta) \leq \dfrac{u}{y}   \overline{z}(\delta)) =  \dfrac{u}{y} (1 \rightarrow_{GG} \overline{z}(\delta)).$$
As we have $P(u , x , y , z) = \sigma_{GG}(\dfrac{u}{y} , x , 1 , z) \leq \delta$, by (\ref{eq:trucpiqueachebmaxprod}) the inequality  
$\underline{x}(\delta) \leq   \dfrac{u}{y} (1 \rightarrow_{GG} \overline{z}(\delta))$ is verified.
\end{proof}
\noindent We illustrate this result:

\begin{example}
 Let $u=0.3, x=0.6, y = 0.4$ and $z=0.2$. We compute $\delta := P(0.3,0.6,0.4,0.2) = \dfrac{xy - uz}{u + y}$ then $\dfrac{ y \, \underline{x}(\delta)}{u}  = \overline{z}(\delta)$ so the inequality is satisfied with $\delta \simeq 0.257$.

\end{example}
We establish  the following version of (Lemma \ref{lemma:equivmin})  for the $\min-\rightarrow_{GG}$ system: 
\begin{lemma}\label{lemma:sametrucformingg}
Let $j\in \{1 , 2 , \dots , m\}$. 
for all $\delta \in [0 , 1]$ such that $\delta  < 1 - \beta_j$, we have the following equivalence between the following two statements:
\begin{enumerate}
\item $G(\underline{\beta}(\delta))_j \leq  \overline{\beta}(\delta)_j$, 
\item there exists $i\in \{1 , 2 , \dots, n\}$ such that:
$$ \gamma_{ji} > 0, \quad  \theta_{ji}   < \delta, \quad   \zeta_{ji}  \leq  \delta.$$ 
\end{enumerate}
\end{lemma}  
\begin{proof}\mbox{}\\
$\Longrightarrow$:\\ 
We have:\\
$$G(\underline{\beta}(\delta))_j  = \min_{1 \leq i \leq n}\gamma_{ji} \rightarrow_{GG} [\max_{1 \leq l \leq m} \gamma_{li}\cdot \underline{\beta}(\delta)_l]\leq 
\overline{\beta}(\delta)_j = \beta_j + \delta < 1.$$ 
We deduce that there exists $i\in \{1 , 2 , \dots  n\}$ such that:  
$$\gamma_{ji} \rightarrow_{GG} [\max_{1 \leq l \leq m} \gamma_{li}\, \underline{\beta}(\delta)_l] =.
G(\underline{\beta}(\delta))_j   \leq \overline{\beta}(\delta)_j = \beta_j + \delta < 1$$
which implies 
$$\gamma_{ji} > \max_{1 \leq l \leq m} \gamma_{li}\, \underline{\beta}(\delta)_l\quad \text{ and } \quad 
\gamma_{ji} \rightarrow_{GG} [\max_{1 \leq l \leq m} \gamma_{li}\, \underline{\beta}(\delta)_l]=\dfrac{1}{\gamma_{ji}} \max_{1 \leq l \leq m} \gamma_{li}\, \underline{\beta}(\delta)_l\leq  \beta_j + \delta$$
so we have
\begin{equation}\label{qu:eq1}\gamma_{ji} > 0\quad  \text{and}
\quad  \forall l\in \{1 , 2 , \dots , m\}\,,\, \gamma_{ji} >  \gamma_{li}\, \underline{\beta}(\delta)_l \quad \text{ and }\quad     \dfrac{\gamma_{li}\, \underline{\beta}(\delta)_l}{\gamma_{ji}} \leq  \beta_j + \delta = \overline{\beta}(\delta)_j < 1.\end{equation}

\noindent By (Lemma \ref{l3}), since $\delta < 1 - \beta_j$, we know that for all $l\in\{1 
 , 2 , \dots ,m\}$, we have:
 $$\dfrac{\gamma_{li}\, \underline{\beta}(\delta)_l}{\gamma_{ji}} \leq  \beta_j + \delta =  \overline{\beta}(\delta)_j
 \Longleftrightarrow
 P(\gamma_{ji} , \beta_l , \gamma_{li} , \beta_j)  \leq \delta$$
so $\zeta_{ji} := \max_{1 \leq l \leq m} P(\gamma_{ji} , \beta_l , \gamma_{li} , \beta_j) \leq \delta$. 

To establish the inequality  $\theta_{ji} < \delta$, we just have to show that we have:
$$\forall l\in W(j , i) \, , \,   \beta_l -   \dfrac{\gamma_{ji}}{\gamma_{li}} <  \delta.   $$
(We have $W(j , i) \not=\emptyset$, because $\gamma_{ji} > 0$, so $j\in W(j , i)$)

Let  $l\in W(j , i)$ i.e.,   
  $0 
 < \gamma_{ji} \leq \gamma_{li}$, and we deduce from (\ref{qu:eq1}):    
$$     \dfrac{\gamma_{li}}{\gamma_{ji}}\, \underline{\beta}(\delta)_l = \max(\dfrac{\gamma_{li}}{\gamma_{ji}} \beta_l - \dfrac{\gamma_{li}}{\gamma_{ji}} \delta , 0 ) < 1.$$  
Therefore, we have:
$$\dfrac{\gamma_{li}}{\gamma_{ji}} \beta_l - \dfrac{\gamma_{li}}{\gamma_{ji}} \delta < 1 $$
so $\beta_l  -   \dfrac{\gamma_{ji}}{\gamma_{li}}  < \delta$ and 
$\theta_{ji} : = \max_{l\in W(j , i)} (\beta_l  -   \dfrac{\gamma_{ji}}{\gamma_{li}})  < \delta.$
\noindent

$\Longleftarrow$: \\ 
Let $i\in \{1 , 2 , \dots, n\}$ such that 
$$ \gamma_{ji} > 0, \quad  \theta_{ji}   < \delta, \quad   \zeta_{ji}  \leq  \delta.$$
We have:  
$$G(\underline{\beta}(\delta))_j  = \min_{1 \leq i' \leq n}\gamma_{ji'} \rightarrow_{GG} [\max_{1 \leq l \leq m} \gamma_{li'}\, \underline{\beta}(\delta)_l]
\leq  \gamma_{ji} \rightarrow_{GG} [\max_{1 \leq l \leq m} \gamma_{li}\, \underline{\beta}(\delta)_l]$$
and it suffices to show that:
$$\gamma_{ji} \rightarrow_{GG} [\max_{1 \leq l \leq m} \gamma_{li}\, \underline{\beta}(\delta)_l] \leq \overline{\beta}(\delta)_j    $$
From the general equality    $x \rightarrow_{GG} \max(y , z) = \max(x \rightarrow_{GG} y , x \rightarrow_{GG} z)$, we deduce:
$$\gamma_{ji} \rightarrow_{GG} [\max_{1 \leq l \leq m} \gamma_{li}\, \underline{\beta}(\delta)_l]  =  \max_{1 \leq l \leq m}  \gamma_{ji}  \rightarrow _{GG}  \gamma_{li}\, \underline{\beta}(\delta)_l.$$
We will establish that 
\begin{equation}\label{eq:mingggammajiineqproof}
    \forall l \in\{1 , 2 , \dots , m\}\,,\, \gamma_{ji}  \rightarrow_{GG}  \gamma_{li}\, \underline{\beta}(\delta)_l \leq 
\overline{\beta}(\delta)_j   = \beta_j + \delta.
\end{equation}

\noindent
Let $l \in\{1 , 2 , \dots , m\}$. We distinguish the following two cases:
\begin{itemize}
    \item Suppose that $l\in W(j , i)$, i.e., $  \gamma_{ji} \leq \gamma_{li}$. We have   $\gamma_{li} > 0$. Then, we have 
$$\beta_l -  \dfrac{\gamma_{ji}}{\gamma_{li}}   \leq  \theta_{ji} < \delta.$$
We deduce: 
$$\dfrac{\gamma_{li}}{\gamma_{ji}}\beta_l -  \dfrac{\gamma_{li}}{\gamma_{ji}} \delta < 1$$.
so 
$$\dfrac{\gamma_{li}}{\gamma_{ji}}\underline{\beta}(\delta)_l          = \max(\dfrac{\gamma_{li}}{\gamma_{ji}}\beta_l  - \dfrac{\gamma_{li}}{\gamma_{ji}} \delta , 0)  < 1$$
which leads to 
$$\gamma_{ji}  \rightarrow_{GG}  \gamma_{li}\, \underline{\beta}(\delta)_l = \dfrac{\gamma_{li}}{\gamma_{ji}} \underline{\beta}(\delta)_l.$$
By the hypothesis, we have $P(\gamma_{ji}  , \beta_l , \gamma_{li} , \beta_j) \leq   \zeta_{ji} \leq \delta$.

As  $\gamma_{ji} > 0$ and $\delta < 1 - \beta_j$ , by (Lemma \ref{l3}), we deduce: 
$$\dfrac{\gamma_{li}}{\gamma_{ji}} \underline{\beta}(\delta)_l \leq \beta_j + \delta.$$
So, in this case, (\ref{eq:mingggammajiineqproof}) is proved.
\item Suppose that $l\in W(j , i)^c$ i.e. $  \gamma_{ji} >  \gamma_{li}$.  Since $\gamma_{li} \geq \gamma_{li} \underline{\beta}(\delta)_l$,  we  have:  
$$\gamma_{ji}  \rightarrow_{GG}  \gamma_{li}\, \underline{\beta}(\delta)_l = \dfrac{\gamma_{li}}{\gamma_{ji}} \underline{\beta}(\delta)_l.$$
As in the previous case, by the inequality  $P(\gamma_{ji}  , \beta_l , \gamma_{li} , \beta_j) \leq   \zeta_{ji} \leq \delta$ and (Lemma \ref{l3}) we then deduce:
$$\gamma_{ji}  \rightarrow _{GG}  \gamma_{li}\, \underline{\beta}(\delta)_l = \dfrac{\gamma_{li}}{\gamma_{ji}} \underline{\beta}(\delta)_l \leq \beta_j + \delta.$$
So, we proved (\ref{eq:mingggammajiineqproof}).
\end{itemize}

\end{proof}

\noindent

In the previous section, for the $\min-\rightarrow_G$ system,  (Lemma \ref{lemma:equivmin}) has allowed us to prove the results stated in (Proposition \ref{proposition:minusunmoinsbetamaxthetazeta}), (Corollary \ref{corollary:c1}), (Proposition \ref{enplus1}) and (Theorem \ref{theorem:chebming}).  Thanks to the choice of notations $E_j$, $\nabla_j$, $\theta_{ji}$ and $\zeta_{ji}$ which are used for the $\min-\rightarrow_G$ system and the $\min-\rightarrow_{GG}$ system, and  using (Lemma \ref{lemma:sametrucformingg}), it is easy to establish the same results stated in  (Proposition \ref{proposition:minusunmoinsbetamaxthetazeta}), (Corollary \ref{corollary:c1}), (Proposition \ref{enplus1}) and (Theorem \ref{theorem:chebming}) for the  $\min-\rightarrow_{GG}$ system. To avoid unnecessary repetition, we only state a new theorem below (Theorem \ref{thformingg}), which corresponds to (Theorem \ref{theorem:chebming}) for the $\min-\rightarrow_{GG}$ system. Both theorems have the same proof.

We then give the explicit analytical formula for computing the Chebyshev distance $\nabla$ (see (Definition \ref{def:chebyshevdistMinfleche})   and  (Proposition \ref{prop:nablafunc})), associated to the second member $\beta$ of a system of $\min-\rightarrow_{GG}$ fuzzy relational equations $(\Sigma) : \Gamma  \Box_{\rightarrow_{GG}}^{\min} x =  \beta$, see (\ref{eq:minsys}):
\begin{theorem}\label{thformingg}
$$\nabla = \max_{1 \leq j \leq m} \nabla_j \quad  \text{where for }  j \in \{1,2,\dots,m\}, \quad  \nabla_j = 
\min(1 - \beta_j , \tau_j)$$
\noindent $$\text{ with } \quad
\tau_j = \min_{i\in {\cal A}_j}\, \max(\theta_{ji}, \zeta_{ji}) \quad \text{ and the convention }  \min_\emptyset = 1.$$ 
\end{theorem}

We illustrate this theorem:
\begin{example}
Let us reuse the matrix $\Gamma = \begin{bmatrix}
    0.6 & 0.49 \\ 
    0.26 & 0.9 \\ 
\end{bmatrix}$ and the vector $\beta = \begin{bmatrix}
    0.1\\
    0.4
\end{bmatrix}$. For $j=1$, we have:

\begin{align*}
    G(\underline{\beta}(\delta))_1 &= \min_{1 \leq i \leq 2}(\gamma_{1i}  \rightarrow_{GG} (\max_{1 \leq l \leq m}\gamma_{li}  \cdot  \underline{\beta}(\delta)_l))\\
&=\min(0.6  \rightarrow_{GG} (\max_{1 \leq l \leq 2}\gamma_{l1}  \cdot  \underline{\beta}(\delta)_l), 0.49  \rightarrow_{GG} (\max_{1 \leq l \leq 2}\gamma_{l2}  \cdot  \underline{\beta}(\delta)_l)) 
\end{align*}
\noindent We compute:
\begin{itemize}
    \item $\max_{1 \leq l \leq 2} \gamma_{l1} \cdot {\beta}_l = \max(0.6 \cdot 0.1, 0.26 \cdot 0.4)=0.10$,
    \item $\max_{1 \leq l \leq 2} \gamma_{l2}  \cdot  {\beta}_l = \max(0.49 \cdot 0.1, 0.9 \cdot 0.4)=0.36$
\end{itemize}
\noindent and then: $0.6 \rightarrow_{GG} 0.10 =0.16$ and $0.49 \rightarrow_{GG} 0.36=0.73$. So the inequality $G(\underline{\beta}(\delta))_1 \leq 
\overline{\beta}(\delta)_1$ is not satisfied for $\delta = 0$. 

\noindent For solving the inequality, we rely on: $W(1,1) = \{1\}$ and $W(1,2) = \{ 1,2\}$ and we compute: 
\begin{itemize}
    \item $\theta_{11} = \beta_1 - \dfrac{\gamma_{11}}{\gamma_{11}} = 0.1 - 1 = -0.9$, 
    \item $\theta_{12} = \max(\beta_1 -  \dfrac{\gamma_{12}}{\gamma_{12}}, \beta_2 -  \dfrac{\gamma_{12}}{\gamma_{22}})=\max(0.1-1,0.4-\dfrac{0.49}{0.9})=\max(-0.9, -0.14) = -0.14$,
    \item $\zeta_{11} = \max(P(0.6, 0.1,0.6,0.1),P(0.6, 0.4,0.26,0.1))= 0.05$, 
    \item $\zeta_{12} = \max(P(0.49, 0.1, 0.49,0.1), P(0.49,0.4, 0.9,0.1)) = 0.22$.  
\end{itemize}
 
\noindent We compute 
$\tau_1 = \min(0.05 , 0.22) = 0.05$ so we have $\nabla_1 = \min(0.9,0.05) = 0.05$.

\noindent Then we observe that $ G(\underline{\beta}(\nabla_1))_1 = \min(0.6 \rightarrow_G 0.09, 0.49 \rightarrow_G 0.31) = 0.15 \leq \overline{\beta}(\nabla_1)_1 = 0.15$. So the inequality is solved with $\nabla_1 = 0.05$.
\end{example}

\noindent
Although for the t-norm product, its associated residual implicator $\rightarrow_{GG}$ (called Goguen's implication) is not continuous with respect to its two variables, we show in the remainder of this section that the lower bounds $\nabla_j$ and $\nabla$ are in fact minimum elements of the sets $E_j$ and $E$ respectively, i.e., we have $\nabla_j \in E_j$ and $\nabla \in E$. Recall that, by   (Example \ref{ex:5counter}), these properties are not always true for the $\min-\rightarrow_G$ system.  

\noindent
We begin by establishing two useful lemmas to show that we have $\nabla_j \in E_j$ for all $j\in\{1 , 2 , \dots , m\}$.

\begin{lemma}\label{lemma:fracsup}
Let $u , x , y , z \in [0 , 1]$ satisfy the following two conditions:    
\begin{enumerate}
\item $0 < u < y$
\item $0 < x - \dfrac{u}{y} < 1 - z$.
\end{enumerate}
Then, we have :  
$$\dfrac{(xy - u z)^+}{u + y}  = 
\dfrac{ xy - u z }{u + y}> x - \dfrac{u}{y}.$$
\end{lemma}   
\begin{proof}
Since by hypothesis   we have $x - \dfrac{u}{y} > 0$ and  $u + y > 0$, 
 We just have to show the strict inequality 
 $\dfrac{ xy - u z }{u + y}> x - \dfrac{u}{y}$.

\noindent
We have:  
\begin{align}
\dfrac{ x y - u z }{u + y}&= 
\dfrac{ x y + x u - x u - u z }{u + y}\nonumber\\
&= x - \dfrac{u}{u + y}(x + z)\nonumber\\
& > x - \dfrac{u}{u + y}(1 + \dfrac{u}{ y})  \nonumber\\
& = x - \dfrac{u}{ y} > 0.  \nonumber
\end{align}
 
\end{proof}
\begin{lemma}\label{lemma:thetajizetaji}
We have: 
$$\forall i \in {\cal A}_j\,,\, \theta_{ji}  \leq \zeta_{ji}.$$
\end{lemma}
\begin{proof} Let   $i \in {\cal A}_j$, i.e.,  $\gamma_{ji} > 0$, we deduce:  
   $$W(j , i) = \{l\in \{1 , 2 , \dots , m\}\,\mid \, \gamma_{ji} \leq \gamma_{li} \,\,\text{and}\,\, \gamma_{li} > 0\} \not=\emptyset.$$ 
  We recall that  $\theta_{ji}  := \max_{l\in W(j , i)} (\beta_l - \dfrac{\gamma_{ji}}{\gamma_{li}})$ and 
  $\zeta_{ji} = \max_{1 \leq l \leq m}  P(\gamma_{ji} , \beta_l  , \gamma_{li}  , \beta_j)$.

\noindent  
For all $l\in W(j , i)$, let us show the inequality:
$$\beta_l - \dfrac{\gamma_{ji}}{\gamma_{li}} \leq \zeta_{ji}.$$
Since  $\gamma_{ji} > 0$ and  $ \gamma_{li} > 0$, we have: 
\begin{align}
\zeta_{ji} & \geq P(\gamma_{ji} , \beta_l  , \gamma_{li}  , \beta_j) \nonumber\\
& =   \max((\beta_l - \dfrac{\gamma_{ji}}{\gamma_{li}})^+ , \, 
\min(\dfrac{(\beta_l \gamma_{li} - \beta_j \gamma_{ji})^+}{\gamma_{ji} + \gamma_{li}} , 1 - \beta_j)) \quad \text{see (\ref{eq:egalitesursigmagg})}\nonumber\\
& \geq (\beta_l - \dfrac{\gamma_{ji}}{\gamma_{li}})^+\nonumber\\
& \geq  \beta_l - \dfrac{\gamma_{ji}}{\gamma_{li}}. \nonumber
\end{align}
Therefore
$$\zeta_{ji} \geq \max_{l\in W(j , i)} (\beta_l - \dfrac{\gamma_{ji}}{\gamma_{li}}) = \theta_{ji}.$$
 
   \end{proof}

We are now able to show that for $j \in \{1,2,\dots,m\}$, we  have $\nabla_j \in E_j$ i.e., $\nabla_j = \min E_j$, and therefore $\nabla \in E$ i.e., $\nabla = \min E$.

   \begin{theorem}\label{th:minGGisMin} For all  $j\in\{1 , 2 , \dots , m\}$ , we have:    
$$\nabla_j   = \min  E_j\quad \text{ i.e.,} \quad 
\nabla_j \in  E_j.$$
\noindent Therefore:
\begin{equation*}
    \nabla = \min E. 
\end{equation*}
\end{theorem}
\begin{proof}
From (Lemma \ref{lemma:nablajvsunmoinsbeta}), we know that $1 - \beta_j \in E_j$ and $\nabla_j  \leq 1 - \beta_j$. 

\noindent
It remains to show that $\nabla_j = \min E_j$ under the hypothesis $\nabla_j  < 1 - \beta_j$.
 
\noindent
From (Theorem \ref{thformingg}) and (Lemma \ref{lemma:thetajizetaji}) we know:
$$\nabla_j = \min(1  - \beta_j , \tau_j) = \tau_j < 1 \quad 
 \text{and} \quad \tau_j = \min_{i\in{\cal A}_j} \max(\theta_{ji} , \zeta_{ji}) = \min_{i\in{\cal A}_j} \zeta_{ji}.$$
Therefore  
 ${\cal A}_j \not=\emptyset$. Let $i\in {\cal A}_j$ i.e., $\gamma_{ji} > 0$, such that:   
 \begin{equation}\label{qu:eqth}
\nabla_j  = \tau_j = \max(\theta_{ji} , \zeta_{ji}) = \zeta_{ji}.
\end{equation}

\noindent
To show that $\nabla_j \in E_j$, we will prove that $\nabla_j$ satisfies the second assumption of (Lemma \ref{lemma:sametrucformingg}). Taking into account the hypothesis $\nabla_j < 1 - \beta_j$ and the equality $\nabla_j = \zeta_{ji}$, see (\ref{qu:eqth}), we have to show the strict inequality:
\begin{equation}\label{eq:ineqthetazeta}
    \theta_{ji} < \zeta_{ji}.
\end{equation} 
Let us prove this strict inequality $ \theta_{ji} < \zeta_{ji} $ in the following two steps:

\noindent
$\bullet \,$ If  $\nabla_j = 0$ : we must prove that  $\theta_{ji} < 0$.

\noindent Since $W(j , i) \neq \emptyset$ (we have $j \in W(j , i)$), let $l\in W(j , i)$ be such that   
$\theta_{ji}  = \beta_l - \dfrac{\gamma_{ji}}{\gamma_{li}}$.
We have 
\begin{align}
0  = \nabla_j = & \zeta_{ji}\nonumber\\
 &\geq P(\gamma_{ji} , \beta_l  ,  \gamma_{li}  , \beta_j) \nonumber\\
& =   \max( (\beta_l - \dfrac{\gamma_{ji}}{\gamma_{li}})^+ , 
\min(\dfrac{(\beta_l \gamma_{li} - \beta_j \gamma_{ji})^+}{\gamma_{ji} + \gamma_{li}} , 1 - \beta_j)) \quad \text{see (\ref{eq:egalitesursigmagg})}\nonumber   
\end{align}
Since  $    1 - \beta_j >   \nabla_j = 0$, we deduce  $\dfrac{(\beta_l \gamma_{li} - \beta_j \gamma_{ji})^+}{\gamma_{ji} + \gamma_{li}} = 0$, which implies  
$\beta_l \gamma_{li} \leq \beta_j \gamma_{ji}$.  We have:  
\begin{align}
  \theta_{ji} & =  \beta_l - \dfrac{\gamma_{ji}}{\gamma_{li}}\nonumber\\
 & \leq \dfrac{\gamma_{ji}}{\gamma_{li}} \beta_j - \dfrac{\gamma_{ji}}{\gamma_{li}}\nonumber\\
 &= \dfrac{\gamma_{ji}}{\gamma_{li}}(\beta_j - 1) < 0.\nonumber
\end{align}

\noindent
$\bullet \,$ If $\nabla_j > 0$: let us prove by contradiction the strict inequality $\theta_{ji} < \zeta_{ji}$. \\
Assume we have  
$\theta_{ji} \geq  \zeta_{ji}.$ Then, by (Lemma \ref{lemma:thetajizetaji}) we have: $$\theta_{ji} =  \zeta_{ji}.$$   
Let  $l\in W(j , i) $ be such that $\theta_{ji} = \beta_l - \dfrac{\gamma_{ji}}{\gamma_{li}}$. So we  have 
 \begin{equation}\label{eq:proofth4nablaj}
     \nabla_j = \zeta_{ji} = \theta_{ji} = \beta_l - \dfrac{\gamma_{ji}}{\gamma_{li}} > 0.
 \end{equation} 
   We then have: 
\begin{align}
 1 - \beta_j &> \nabla_j \nonumber\\
 &= \beta_l - \dfrac{\gamma_{ji}}{\gamma_{li}} \nonumber\\
 &= \zeta_{ji}\nonumber\\
 &\geq P(\gamma_{ji} , \beta_l  , \gamma_{li}  , \beta_j) \quad \text{see (\ref{eq:egalitesursigmagg})} \nonumber\\
& =   \max( \beta_l - \dfrac{\gamma_{ji}}{\gamma_{li}} , 
\min(\dfrac{(\beta_l \gamma_{li} - \beta_j \gamma_{ji})^+}{\gamma_{ji} + \gamma_{li}} , 1 - \beta_j)).\nonumber
\end{align}
Therefore:
$$\min(\dfrac{(\beta_l \gamma_{li} - \beta_j \gamma_{ji})^+}{\gamma_{ji} + \gamma_{li}} , 1 - \beta_j) \leq \beta_l - \dfrac{\gamma_{ji}}{\gamma_{li}} < 1 - \beta_j.$$
This last inequality implies: 
$$\dfrac{(\beta_l \gamma_{li} - \beta_j \gamma_{ji})^+}{\gamma_{ji} + \gamma_{li}} \leq \beta_l - \dfrac{\gamma_{ji}}{\gamma_{li}} < 1 - \beta_j.$$
We put     $u = \gamma_{ji} , x  = \beta_l , y = \gamma_{li} , z = \beta_j$. By the previous calculations, we have: 
\begin{itemize}
    \item $0 < u < y$ because  $l\in W(j , i)$ and  $\beta_l - \dfrac{\gamma_{ji}}{\gamma_{li}} = \nabla_j > 0$,
    \item $0 < x - \dfrac{u}{y } = \beta_l - \dfrac{\gamma_{ji}}{\gamma_{li}} < 1 - \beta_j = 1 - z$.
\end{itemize}

\noindent Applying (Lemma \ref{lemma:fracsup}), we get: 
$$ \dfrac{(xy - u z)^+}{u + y} > x - \dfrac{u}{y}.$$
We replace $u , x , y , z$   with their values. This leads to the following inequality:  
$$\dfrac{(\beta_l \gamma_{li} - \beta_j \gamma_{ji})^+}{\gamma_{ji} + \gamma_{li}} > \beta_l - \dfrac{\gamma_{ji}}{\gamma_{li}}$$
which contradicts
$$\dfrac{(\beta_l \gamma_{li} - \beta_j \gamma_{ji})^+}{\gamma_{ji} + \gamma_{li}} \leq \beta_l - \dfrac{\gamma_{ji}}{\gamma_{li}}.$$
We have therefore shown the strict inequality 
$\theta_{ji}  < \zeta_{ij}$, see (\ref{eq:ineqthetazeta}). Therefore $\nabla_j \in E_j$. 

\noindent As a consequence, since we proved that for all $j \in \{1,2,\dots,m\}$ we have $\nabla_j \in E_j$,  we deduce from  (Lemma \ref{lemdec}) and (Proposition \ref{prop:nablafunc}) that $\nabla = \min E \text{ i.e., } \nabla \in E$. 
\end{proof}

The equality $\nabla = \min E$ allows us to prove directly that the set of Chebyshev approximations $\mathcal{D}_\beta$, see (\ref{eq:setDb}), always has a minimum element. We also give an approximate solution: 
\begin{corollary}\label{cor:minGGcheb}
$G(\underline{\beta}(\nabla))$ is the lowest Chebychev approximation for $\beta$. Moreover, $\xi = \Gamma^t \Box_{T_P}^{\max} \underline{\beta}(\nabla)$ is an approximate solution for the system 
$(\Sigma) : \Gamma \Box_{\rightarrow_{GG}}^{\min} x = \beta$.
\begin{proof}
Same proof as (Corollary \ref{cor:chebnablainE}) but using  the $\min-\rightarrow_{GG}$ composition. 
\end{proof}
\end{corollary}
By definition of $\nabla$ (Definition \ref{def:chebyshevdistMinfleche}) and the above corollary,  it is clear that: 
\begin{corollary}\label{cor:minGGmin}  We have:   
$$\nabla    = \min_{d\in {\cal D}} \Vert \beta - d\Vert. 
$$
and the set of Chebyshev approximations $\mathcal{D}_\beta$,  see (\ref{eq:setDb}), is non-empty. 
\end{corollary}

\section{\texorpdfstring{Chebyshev distance associated to the second member of systems based on $\min-\rightarrow_L$ composition}{Chebyshev distance associated to the second member of a system of min→Łukasiewicz fuzzy relational equations}}
\label{sec:minluka}

 In this section,  we study the Chebyshev distance $\nabla$, see (Definition \ref{def:chebyshevdistMinfleche}), associated to the second member of a system of $\min-\rightarrow_L$ fuzzy relational equations $(\Sigma) : \Gamma  \Box_{\rightarrow_{L}}^{\min} x =  \beta$,  see (\ref{eq:minsys}), where $\rightarrow_L$ is the Łukasiewicz implication, see (\ref{eq:tnormluka}).

To establish  the formula for computing the Chebyshev distance $\nabla$, we prove a version of (Lemma \ref{lemma:equivmin}) (or (Lemma \ref{lemma:sametrucformingg})) for the $\min-\rightarrow_{L}$ system (Lemma \ref{lemma:mainlemmeforminL}). In (Lemma \ref{lemma:mainlemmeforminL}), we have retained the two notations $\nabla_j$ and $\zeta_{ji}$ from previous sections with their own definitions for the $\min-\rightarrow_{L}$ system (but not the  notation $\theta_{ji}$, which is not needed).

Since Łukasiewicz's implication is continuous with respect to its two variables, establishing the formula for $\nabla$ (Theorem \ref{th:luka}), is easier than for $\min-\rightarrow_{G}$ and $\min-\rightarrow_{GG}$ systems.

\noindent We remind that we have $\nabla = \max_{1\leq j \leq m} \nabla_j$ where $\nabla_j = \inf E_j$ (Proposition \ref{prop:nablafunc}) and the set $E_j$ is defined in (\ref{eq:setEj}). Unlike the $\min - \rightarrow_{G}$ system, we will show that in the case of the $\min - \rightarrow_{L}$ system (as it is the same for the $\min-\rightarrow_{GG}$ system), for all $j \in \{1,2,\cdots,m\}$, we always have $\nabla_j \in E_j$, i.e., $\nabla_j$ is always the minimum element of the set $E_j$.  By applying (Lemma \ref{lemdec}), we obtain  $\nabla = \min E$, where $E$ is defined in (\ref{qu:ETot}). Therefore, for a $\min-\rightarrow_{L}$ system, the Chebyshev distance $\nabla$ is always a minimum, see (Corollary \ref{cor:minLmin}).  In fact, we can always compute the lowest Chebyshev approximation of the second member
and an approximate solution of the system $(\Sigma)$, see (Corollary \ref{cor:minLcheb}).

 We rely on the following function:
\begin{equation}
        L(u , v , x , y) = \max((u - y)^+ , \min((x -  v)^+ , \dfrac{(x - y + u - v)^+)}{2})). 
\end{equation}
and we introduce  new notations:
\begin{notation}\label{not:zjiluka}
For $1 \leq j \leq m, 1 \leq i \leq n$, to each coefficient $\gamma_{ji}$ of the matrix $\Gamma$ of the system $(\Sigma)$,  see (\ref{eq:minsys}), we associate: $$\zeta_{ji} = \max_{1 \leq l \leq m}L(1 - \gamma_{ji} , 1 -  \gamma_{li} , \beta_l , \beta_j).$$
\end{notation}

\noindent The function $L$ allows us to solve an inequality:

\begin{lemma}\label{l4}
Let  $u , v, x, y $ in $[0 , 1]$. For all  $\delta\in [0 , 1 - y[$, we have the following equivalence:  
\begin{equation}\label{eq:L4}
    (\underline{x}(\delta) - v)^+ \leq \overline{y}(\delta)  - u \,\Longleftrightarrow \,
\delta \geq L(u , v , x  , y).
\end{equation}
\end{lemma} 
\begin{proof}\mbox{}\\
$\Longrightarrow$:\\

\noindent
$\bullet$ We have $\delta \geq (u - y)^+$. Indeed,  we deduce from the inequality of the left side of (\ref{eq:L4}) and the hypothesis $y + \delta < 1$ that we have:
\[0 \leq \overline{y}(\delta)  - u = \min(y + \delta -  u , 1 - u) = y + \delta - u, 
\] 
so  
$\delta \geq u - y$. As $\delta \geq 0$, we obtain that $\delta \geq (u - y)^+$. 

\noindent
$\bullet$ Let us prove $\delta \geq \min((x -  v)^+ , \dfrac{(x - y + u - v)^+)}{2})$.

\noindent
The associativity of the function $\max$ (i.e,  $\max(x  , y , z) = \max(\max(x , y) , z)$ leads to the equality: 
\begin{equation}\label{eq:maxxdeltaminusv}
    (\underline{x}(\delta) - v)^+ =  (x - \delta  - v)^+.
\end{equation}
Indeed, we have:
\begin{align*}
(\underline{x}(\delta) - v)^+ &= \max(\underline{x}(\delta) - v , 0) \nonumber\\
& = \max(\underline{x}(\delta) , v) - v  \nonumber\\
& = \max(\max(x - \delta , 0) , v) - v\nonumber\\
& = \max(x - \delta , 0 , v)   - v\nonumber\\
& = \max(x - \delta  - v ,  -v , 0)   \nonumber\\
& = \max((x - \delta  - v)^+,  - v )  =  (x - \delta  - v)^+\nonumber
\end{align*}
We then deduce:
\begin{align}\label{eq:equivdanslemma12}
(\underline{x}(\delta) - v)^+ = 0 
&\Longleftrightarrow (x - \delta  - v)^+ =0 \nonumber\\
&\Longleftrightarrow  x - \delta  \leq  v \nonumber\\
&\Longleftrightarrow  x - v  \leq  \delta \nonumber\\
&\Longleftrightarrow  (x -  v)^+ \leq  \delta. 
\end{align}
It is clear that we have:  
$$ \delta \geq (x -  v)^+ \Longrightarrow \delta \geq \min((x -  v)^+ , \dfrac{(x - y + u - v)^+)}{2}).$$
Suppose now that $\delta < (x -  v)^+$. Then by (\ref{eq:equivdanslemma12}), we have:
$$0 < (\underline{x}(\delta) - v)^+ = \underline{x}(\delta) - v.$$
and the hypothesis becomes:
$$0 < \underline{x}(\delta) - v \leq \overline{y}(\delta) - u = 
y + \delta  - u$$
The inequality $\underline{x}(\delta) - v \leq \overline{y}(\delta) - u = 
y + \delta  - u$ is equivalent to
$$  \max(x - \delta , 0) - v - [y + \delta  - u]
= \max(x  - y + u - v  - 2\delta , - v - [y + \delta  - u]) \leq 0.$$
so $  x  - y + u - v  - 2\delta \leq 0$. As $\delta \geq 0$, we also have $   (x - y + u - v)^+ \leq 2\delta$. Finally, we have:
 $$ \delta \geq  \dfrac{(x - y + u - v)^+}{2} \geq \min((x -  v)^+ , \dfrac{(x - y + u - v)^+}{2}).$$

\noindent To summarize, we have proven that $\delta \geq \max((u - y)^+ , \min((x -  v)^+ , \dfrac{(x - y + u - v)^+)}{2})) = L(u,v,x,y).$
 
 $\Longleftarrow:$ \\ 
 We must show the inequality: $(\underline{x}(\delta) - v)^+ \leq \overline{y}(\delta)  - u.$\\
 \noindent We note that, by hypothesis, we have $\overline{y}(\delta) = y + \delta < 1$. \\

\noindent $\bullet\,$ Let us show that $\overline{y}(\delta)  - u = y + \delta -  u$ is positive. In fact, we deduce from the hypothesis:
\begin{equation}\label{eq:numberneg1}
    \delta \geq (u - y)^+ \geq u - y\quad \text{ so }\quad y + \delta -  u  \geq 0.
\end{equation}

\noindent $\bullet $ If  $(\underline{x}(\delta) - v)^+ = 0$, the inequality $(\underline{x}(\delta) - v)^+ \leq \overline{y}(\delta)  - u$ is clear.

$\bullet$ Suppose that  $(\underline{x}(\delta) - v)^+ > 0$. We deduce from (\ref{eq:maxxdeltaminusv}) that:

$$    0 < (x - \delta -v)^+ = x - \delta - v \quad \text{so} \quad  \delta < x - v= (x-v)^+.
$$

\noindent As we suppose that 
$ \delta \geq \max((u - y)^+ , \min((x -  v)^+ , \dfrac{(x - y + u - v)^+}{2}))$, we deduce: 
\begin{equation}\label{eq:numberneg}\delta \geq \dfrac{(x - y + u - v)^+}{2} \geq \dfrac{ x - y + u - v}{2}.\end{equation}

\noindent
 Let us show now the inequality $(\underline{x}(\delta) - v)^+ \leq \overline{y}(\delta)  - u$. We have:
\begin{align*}
(\underline{x}(\delta) - v)^+ -  [\overline{y}(\delta)  -  u]
&=   \underline{x}(\delta) - v  - [y  + \delta - u] \nonumber\\
&= \max(x - y + u  - v  -  2\delta ,  - v  - [y  + \delta - u]) \leq 0, 
\end{align*}
because the numbers $x - y + u  - v  -  2\delta$ and  $ - v  - [y  + \delta - u]$ are negative (see (\ref{eq:numberneg1}) and (\ref{eq:numberneg})).
 
\end{proof}
We illustrate this result:
\begin{example}
Let $u = 0.3, v = 0.4, x = 0.5, y = 0.2$. We have $0.1= x- v > y - u = -0.1$. We compute $\delta = L(u,v,x,y) = 0.1$ then we have $\delta = 0.1  < 1 - y = 0.8$ and the inequality 
    $(\underline{x}(\delta) - v)^+ \leq \overline{y}(\delta)  - u$ is satisfied. 
\end{example}

\noindent
The version of (Lemma \ref{lemma:equivmin}) (or  (Lemma \ref{lemma:sametrucformingg})) for the  $\min - \rightarrow_L$ system is:
 
\begin{lemma}\label{lemma:mainlemmeforminL}
Let $j\in \{1 , 2 , \dots , m\}$. 
For all $\delta \in [0 , 1]$ such that $\delta  < 1 - \beta_j$,  there is an equivalence between the following two statements:
\begin{enumerate}
\item $G(\underline{\beta}(\delta)_j \leq  \overline{\beta}(\delta)_j$
\item There exists $i\in \{1 , 2 , \dots, n\}$ such that:
$$ \delta \geq \zeta_{ji}.$$ 
\end{enumerate}
\end{lemma}  
\begin{proof}\mbox{}\\
$\Longrightarrow:$ \\
We have:  
\begin{align}
G(\underline{\beta}(\delta))_j & =
\min_{1 \leq i \leq n}\, \gamma_{ji} \rightarrow_L [\max_{1 \leq l \leq m}\,
T_L(\gamma_{li} , \underline{\beta}(\delta)_l)] \nonumber\\
& =
\min_{1 \leq i \leq n}\, \min(1 - \gamma_{ji}  + [\max_{1 \leq l \leq m}\,
 (\gamma_{li}  +  \underline{\beta}(\delta)_l - 1)^+] , 1)\nonumber \end{align}
The hypothesis implies that we have an index $i\in\{1 , 2 , \dots , n\}$ such that:  
$$G(\underline{\beta}(\delta))_j =   \min(1 - \gamma_{ji}  + [\max_{1 \leq l \leq m}\,
 (\gamma_{li}  +  \underline{\beta}(\delta)_l - 1)^+] , 1) \leq \overline{\beta}(\delta)_j.$$
As we suppose that $\beta_j  + \delta < 1$, we have   $\overline{\beta}(\delta)_j = \beta_j  + \delta < 1$, and we deduce:     
$$G(\underline{\beta}(\delta))_j  = 1 - \gamma_{ji}  + [\max_{1 \leq l \leq m}\,
 (\gamma_{li}  +  \underline{\beta}(\delta)_l - 1)^+] \leq \overline{\beta}(\delta)_j < 1.$$
For all $l\in\{1 , 2 , \dots ,m\}$, we deduce: 
$$(\gamma_{li}  +  \underline{\beta}(\delta)_l - 1)^+ \leq   \overline{\beta}(\delta)_j  - (1 - \gamma_{ji}).$$ 
Let us show that $\forall  l\in\{1 , 2 , \dots ,m\}\,,\, \delta \geq  L(1  - \gamma_{ji}  , 1  - \gamma_{li}, \beta_l , \beta_j)$. We have for all  $l\in\{1 , 2 , \dots ,m\}$: 
$$(\gamma_{li}  +  \underline{\beta}(\delta)_l - 1)^+ \leq   \overline{\beta}(\delta)_j  - (1 - \gamma_{ji}).$$
We put $u = 1  - \gamma_{ji} , v = 1 - \gamma_{li}  , x = \beta_l , y = \beta_j$. With these notations, we can easily check that we have:
\[\gamma_{li}  +  \underline{\beta}(\delta)_l - 1 =  \underline{x}(\delta) - v   \quad \text{and}\quad  \overline{\beta}(\delta)_j  - (1 - \gamma_{ji}) = \overline{y}(\delta)  - u.\quad  \]
The previous inequality is rewritten as:
$$ (\underline{x}(\delta) - v)^+ \leq   \overline{y}(\delta)  - u$$
and we also have $\delta < 1 - \beta_j = 1 - y$.

\noindent
By applying (Lemma \ref{l4}), we obtain  $\delta \geq L(u , v , x , y) = L(1  - \gamma_{ji}  , 1  - \gamma_{li}  , \beta_l , \beta_j)$.

\noindent
We have shown $\delta  \geq  \max_{1 \leq l \leq m}\, L(1  - \gamma_{ji}  , 1  - \gamma_{li}  , \beta_l , \beta_j) = \zeta_{ji}$.\\

$\Longleftarrow$:\\
Suppose we have an index $i\in\{1 , 2 , \dots , n\}$ such that: 
$$\delta \geq \zeta_{ji} = \max_{1 \leq l \leq m}\, L(1  - \gamma_{ji}  , 1  - \gamma_{li}  , \beta_l , \beta_j).$$
We have: 
$$G(\underline{\beta}(\delta))_j = 
\min_{1 \leq i' \leq n}\, \min(1 - \gamma_{ji'}  + [\max_{1 \leq l \leq m}\,
 (\gamma_{li'}  +  \underline{\beta}(\delta)_l - 1)^+] , 1) \leq   \min(1 - \gamma_{ji}  + [\max_{1 \leq l \leq m}\,
 (\gamma_{li}  +  \underline{\beta}(\delta)_l - 1)^+] , 1).$$
Let us show the inequality 
\begin{equation}
    \label{eq:ineqdanslemm13}
    \min(1 - \gamma_{ji}  + [\max_{1 \leq l \leq m}\,
 (\gamma_{li}  +  \underline{\beta}(\delta)_l - 1)^+] , 1) \leq \overline{\beta}(\delta)_j,
\end{equation}
which will lead to the inequality to be proved.\\
We have:
\[ 1 - \gamma_{ji} + 
[\max_{1 \leq l \leq m}\,
 (\gamma_{li}  +  \underline{\beta}(\delta)_l - 1)^+] = \max_{1 \leq l \leq m}\,  (1 - \gamma_{ji} + (\gamma_{li}  +  \underline{\beta}(\delta)_l - 1)^+).\]
 We deduce:
\begin{align}
\min(1 - \gamma_{ji}  + [\max_{1 \leq l \leq m}\,
 (\gamma_{li}  +  \underline{\beta}(\delta)_l - 1)^+] , 1) &= \min([\max_{1 \leq l \leq m} \, (1 - \gamma_{ji}  +  
 (\gamma_{li}  +  \underline{\beta}(\delta)_l - 1)^+)] , 1)\nonumber\\
 &= \max_{1 \leq l \leq m}\, \min(1 - \gamma_{ji}  +  
 (\gamma_{li}  +  \underline{\beta}(\delta)_l - 1)^+ , 1).\nonumber
 \end{align}
To prove (\ref{eq:ineqdanslemm13}), it is sufficient to check  that $\forall  l\in\{1 , 2 , \dots ,m\},\, \min(1 - \gamma_{ji}  +  
 (\gamma_{li}  +  \underline{\beta}(\delta)_l - 1)^+ , 1)   \leq \overline{\beta}(\delta)_j $. 
 
\noindent
In fact, we just have to show that for all $l\in\{1 , 2 , \dots ,m\}$, that we have:  
\begin{equation}\label{eq:ineq2lemm13}
    1 - \gamma_{ji}  +  (\gamma_{li}  +  \underline{\beta}(\delta)_l - 1)^+ \leq \overline{\beta}(\delta)_j.
\end{equation} 
Let $l\in\{1 , 2 , \dots , m\}$. By hypothesis, we have:  
$$\delta \geq \zeta_{ji} \geq P(1  - \gamma_{ji}  , 1  - \gamma_{li}  , \beta_l , \beta_j).$$
As $\delta < 1 - \beta_j$, by applying (Lemma \ref{l4}) with  $ u = 1  - \gamma_{ji} , v = 1 - \gamma_{li}  , x = \beta_l , y = \beta_j$, we obtain: 
$$ (\underline{x}(\delta) - v)^+ \leq   \overline{y}(\delta)  - u.$$
Therefore: 
\[(\gamma_{li}  +  \underline{\beta}(\delta)_l - 1)^+ = 
(\underline{x}(\delta) - v)^+ \leq \overline{y}(\delta)  - u = \overline{\beta}(\delta)_j - (1 - \gamma_{ji}) \]
This last inequality is equivalent to the inequality (\ref{eq:ineq2lemm13}). 
\end{proof}
\noindent We prove that for all $j \in \{1,2,\dots,m\}$, the set $E_j$, see (\ref{eq:setEj}), admits  $\nabla_j$ as minimum element: 

\begin{proposition}\label{prop:LukaminEj}
For all $j\in\{1 , 2 , \dots , m\}$, we have $\nabla_j = \min E_j\,\,\text{i.e.,}\,\, \nabla_j \in E_j$.
\end{proposition}
\begin{proof}
For any $\delta \in [0 , 1]$,  we have:  
\[G(\underline{\beta}(\delta))_j = \min_{1 \leq i \leq n}\, \min(1 - \gamma_{ji}  + [\max_{1 \leq l \leq m}\,
 (\gamma_{li}  +  \underline{\beta}(\delta)_l - 1)^+] , 1). \]
Since the functions  $[0 , 1] \rightarrow [0 , 1] : \delta \mapsto \underline{\beta}(\delta)_l$ are    continuous,  we deduce that the function 
\[ [0 , 1] \rightarrow [0 , 1] : \delta \mapsto G(\underline{\beta}(\delta))_j \]
is continuous.

\noindent
Since the function $\delta \mapsto \overline{\beta}_j(\delta)$ is also continuous, by applying general continuity theorems (see Corollary of Theorem 4.8 in \cite{rudin1976}), we deduce that the set $E_j$ is a closed (non-empty) set of $[0 ,1]$, and then $$\nabla_j = \inf E_j = \min E_j.$$
\end{proof}
\begin{corollary}\label{cor:minLnablaegalminE}
    We have: 
    \begin{equation}
        \nabla = \min E.
    \end{equation}
\end{corollary}
\begin{proof}
Consequence of (Proposition \ref{prop:LukaminEj}), (Lemma \ref{lemdec}) and (Proposition \ref{prop:nablafunc}). 
\end{proof}

For all $j\in\{1 , 2 , \dots , m\}$, to establish the formula for computing $\nabla_j$, we put:
\begin{notation}
    $\tau_j := \min_{1 \leq i \leq n}\,  
 \zeta_{ji} \quad \text{where} \quad \zeta_{ji} = \max_{1 \leq l \leq m}L(1 - \gamma_{ji} , 1 -  \gamma_{li} , \beta_l , \beta_j)$ (Notation \ref{not:zjiluka}). 
\end{notation}

We have:
\begin{lemma}\label{lemma:lastlemmaforL}
If $\tau_j < 1 - \beta_j$, then $\tau_j \in E_j$ , so $\nabla_j \leq \tau_j$. 
\end{lemma}
\begin{proof} Let $i \in \{1,2,\dots,n\}$ be such that $\tau_j = \zeta_{ji}$. As we suppose that $\tau_j < 1 - \beta_j$, we remark that $\tau_j$ satisfies the condition of the second statement of (Lemma \ref{lemma:mainlemmeforminL}). So $\tau_j \in E_j$ and then $\nabla_j = \min E_j \leq \tau_j$. 
\end{proof}

We now give the explicit analytical formula for computing the Chebyshev distance $\nabla$ (see (Definition \ref{def:chebyshevdistMinfleche})   and  (Proposition \ref{prop:nablafunc})), associated to the second member $\beta$ of a system of $\min-\rightarrow_L$ fuzzy relational equations $(\Sigma) : \Gamma  \Box_{\rightarrow_{L}}^{\min} x =  \beta$, see (\ref{eq:minsys}):
\begin{theorem}\label{th:luka}
$$\nabla = \max_{1 \leq j \leq m} \nabla_j \quad  \text{where for } \quad j \in \{1,2,\dots,m\}, \quad  \nabla_j = 
\min(1 - \beta_j , \tau_j)$$
$$\text{ with }  \quad 
\tau_j = \min_{1 \leq i \leq n}\,  
 \zeta_{ji}.$$ 
\end{theorem}

\begin{proof}
From (Lemma \ref{lemma:nablajvsunmoinsbeta}), we know that    $\nabla_j \leq  1 -\beta_j$.

\noindent
Let us show the following implication:
$$\nabla_j <  1 - \beta_j \Longrightarrow  \nabla_j = \tau_j.$$
Indeed, if $\nabla_j <  1 - \beta_j$,  as we have $\nabla_j\in E_j$, applying (Lemma \ref{lemma:mainlemmeforminL}), we obtain an index $i\in \{1 , 2 , \dots , n\}$ such that $\nabla_j \geq \zeta_{ji}$. From the definition of $\tau_j$, we deduce:
\[\tau_j \leq  \zeta_{ji} \leq \nabla_j. \]
As $\nabla_j <  1 - \beta_j$, we get $\tau_j  <  1 - \beta_j$. By applying (Lemma \ref{lemma:lastlemmaforL}), we finally obtain $\nabla_j \leq \tau_j$, so $\nabla_j = \tau_j$.

\noindent
Since the inequality $\nabla_j \leq  1 -\beta_j$ is true (Lemma \ref{lemma:nablajvsunmoinsbeta}) and using the two implications 
\[\nabla_j <  1 - \beta_j \Longrightarrow  \nabla_j = \tau_j  \quad \text{and by (Lemma \ref{lemma:lastlemmaforL})} \quad\tau_j < 1 - \beta_j \Longrightarrow \nabla_j \leq \tau_j  \]
lead us to obtain the equality $\nabla_j = 
\min(1 - \beta_j , \tau_j).$

\end{proof}

We illustrate this theorem:
\begin{example}
Let us reuse the matrix $\Gamma = \begin{bmatrix}
    0.6 & 0.49 \\ 
    0.26 & 0.9 \\ 
\end{bmatrix}$ and the vector $\beta = \begin{bmatrix}
    0.1\\
    0.4
\end{bmatrix}$. For $j=1$, we have:

\begin{align*}
    G(\underline{\beta}(\delta))_1 &= \min_{1 \leq i \leq 2}(\gamma_{1i}  \rightarrow_{L} (\max_{1 \leq l \leq 2}T_L(\gamma_{li},  \underline{\beta}(\delta)_l)))\\
&=\min(0.6  \rightarrow_{L} (\max_{1 \leq l \leq 2}T_L(\gamma_{l1},  \underline{\beta}(\delta)_l)), 0.49  \rightarrow_{L} (\max_{1 \leq l \leq 2}T_L(\gamma_{l2},\underline{\beta}(\delta)_l))). 
\end{align*}
\noindent We compute:
\begin{itemize}
    \item $\max_{1 \leq l \leq 2} T_L(\gamma_{l1},{\beta}_l) = \max(T_L(0.6, 0.1), T_L(0.26, 0.4))=0$,
    \item $\max_{1 \leq l \leq 2} T_L(\gamma_{l2},   {\beta}_l) = \max(T_L(0.49, 0.1), T_L(0.9, 0.4))=0.3$
\end{itemize}
\noindent and then: $0.6 \rightarrow_{L} 0 =0.4$ and $0.49 \rightarrow_{L} 0.3=0.81$. So the inequality
$G(\underline{\beta}(\delta))_1 \leq \overline{\beta}(\delta)_1$     is not satisfied for $\delta = 0$. 

\noindent For solving the inequality, we compute: 
\begin{itemize}
    \item $\zeta_{11} = \max(L(1-0.6, 1-0.6,0.1,0.1),L(1-0.6, 1-0.26,0.4,0.1))= 0.3$, 
    \item $\zeta_{12} = \max(L(1-0.49, 1 - 0.49, 0.1,0.1), L(1-0.49,1-0.9, 0.4,0.1)) = 0.41$.  
\end{itemize}
 
\noindent We compute $\tau_1 = \min(0.3 , 0.41) = 0.3$ and 
$\nabla_1 = \min(0.9,0.3) = 0.3$.

\noindent Then we observe that $ G(\underline{\beta}(\nabla_1))_1 = \min(0.6 \rightarrow_L 0, 0.49 \rightarrow_L 0) = 0.4 \leq \overline{\beta}(\nabla_1)_1 = 0.4$. So the inequality is solved with $\nabla_1 = 0.3$.
\end{example}

The equality $\nabla = \min E$ (Corollary \ref{cor:minLnablaegalminE}) allows us to prove directly that the set of Chebyshev approximations $\mathcal{D}_\beta$, see (\ref{eq:setDb}), always has a minimum element. We can also give an approximate solution: 
\begin{corollary}\label{cor:minLcheb}
$G(\underline{\beta}(\nabla))$ is the lowest Chebychev approximation for $\beta$. Moreover, $\xi = \Gamma^t \Box_{T_L}^{\max} \underline{\beta}(\nabla)$ is an approximate solution for the system 
$(\Sigma) : \Gamma \Box_{\rightarrow_{L}}^{\min} x = \beta$.
\begin{proof}
Same proof as for (Corollary \ref{cor:chebnablainE}) but using the $\min-\rightarrow_L$ composition.
\end{proof}
\end{corollary}
By definition of $\nabla$ (Definition \ref{def:chebyshevdistMinfleche}) and the above corollary,  it is clear that: 
\begin{corollary}\label{cor:minLmin}   We have:   
$$\nabla    = \min_{d\in {\cal D}} \Vert \beta - d\Vert. 
$$
and the set of Chebyshev approximations $\mathcal{D}_\beta$ is non-empty. 
\end{corollary}

\section{Conclusion}
\label{sec:conclusion}

In this article, we studied the inconsistency of systems of  $\min-\rightarrow$ fuzzy relational equations. We expressed the Chebyshev distance $\nabla = \inf_{d \in \mathcal{D}} \Vert \beta - d \Vert$ associated to an inconsistent  system of  $\min-\rightarrow$ fuzzy relational equations $\Gamma \Box_{\rightarrow}^{\min} x = \beta$, where $\mathcal{D}$ is the set of second members of consistent systems defined with the same matrix $\Gamma$, as the lower bound of the solutions   of a vector inequality. From this result, we gave explicit analytical formulas for computing the Chebyshev distances for  systems of  $\min-\rightarrow$ fuzzy relational equations $\Gamma \Box_{\rightarrow}^{\min} x = \beta$, where $\rightarrow$ is a residual implicator among the Gödel implication, the Goguen implication or Łukasiewicz’s
implication. An important result is that in the case of the $\min-\rightarrow_{G}$ system, the Chebyshev distance $\nabla$ may be an infimum while for $\min-\rightarrow_{GG}$ and $\min-\rightarrow_{L}$ systems the Chebyshev distance $\nabla$ is always a minimum. For $\min-\rightarrow_{G}$ systems, we added a sufficient condition for $\nabla$ to be a minimum. For $\min-\rightarrow_{GG}$ systems, $\min-\rightarrow_{L}$ systems, or $\min-\rightarrow_{G}$ systems whose Chebyshev distance has been verified as a minimum, we can always compute the lowest Chebyshev approximation of the second member and an approximate solution. However, for the $\min-\rightarrow_{G}$ systems whose Chebyshev distance has been verified as not being a minimum, the set of Chebyshev approximations of its second member is empty.

As perspectives, we may study the structure of the approximate solutions set of a system of $\min-\rightarrow$ fuzzy relational equations $\Gamma \Box_{\rightarrow}^{\min} x = \beta$ and that of the set of Chebyshev approximations of the second member $\beta$ with respect to the Chebyshev distance $\nabla$. For a $\max-\min$  system, the structure of these sets was given in \cite{baaj2023maxmin}. One of the difficulties encountered is the determination of the set of maximal Chebyshev approximations of the second member $\beta$, which may be obtained from the solutions of a particular system of $\min-\rightarrow$ inequalities (as it is done for a $\max-\min$ system in \cite{baaj2023maxmin}). Currently, to our knowledge, there is no extensive study of the solutions set of  systems of $\min-\rightarrow$ inequalities.  Similarly to \cite{baaj2023maxmin},  we may also develop a learning method of approximate weight matrices based on $\min-\rightarrow$ composition according to training data.

As applications, we have the complete solution of the problem of the invertibility of a fuzzy matrix for $\min-\rightarrow$ composition, where $\rightarrow$ is the Gödel implication, the Goguen implication or Łukasiewicz’s
implication: we know the set of matrices which admit a pre-inverse or a post-inverse. In \cite{wu2022analytical}, this problem was tackled in the case of the $\max-\min$ composition.

\bibliographystyle{plainnat} 
\bibliography{ref} 
\end{document}